\documentclass{article}
\PassOptionsToPackage{table}{xcolor}
\usepackage{iclr2027_conference,times}

\usepackage[utf8]{inputenc} 
\usepackage[T1]{fontenc}    
\usepackage{hyperref}       
\usepackage{url}            
\usepackage{xurl}           
\usepackage{newunicodechar} 
\newunicodechar{–}{--}
\newunicodechar{—}{---}
\usepackage{booktabs}       
\usepackage{amsfonts}       
\usepackage{nicefrac}       
\usepackage{microtype}      
\usepackage{xcolor}         
\usepackage{microtype}
\usepackage{graphicx}
\usepackage[table]{xcolor} 
\usepackage{booktabs}      
\usepackage{subcaption}    
\usepackage{threeparttable}
\usepackage{multirow}
\usepackage{hyperref}
\usepackage{amsmath}
\usepackage{amssymb}
\usepackage{mathtools}
\usepackage{amsthm}
\usepackage[capitalize,noabbrev]{cleveref} 

\theoremstyle{plain}
\newtheorem{theorem}{Theorem}[section]
\newtheorem{proposition}[theorem]{Proposition}

\theoremstyle{definition}

\theoremstyle{remark}

\usepackage{algorithm}
\usepackage[endLComment=,italicComments=false]{algpseudocodex}
\definecolor{Gray}{gray}{0.9}
\definecolor{myred}{HTML}{F54254}
\definecolor{myorange}{HTML}{FFB135}
\definecolor{mygreen}{HTML}{10BD35}
\definecolor{myblue}{HTML}{598BE7}
\definecolor{mypurple}{HTML}{9A1C6B}
\definecolor{plgray}{HTML}{999999}
\definecolor{hiccup}{HTML}{001473}
\definecolor{cyl}{HTML}{801dae}
\definecolor{mask}{HTML}{E7F0F9}
\usepackage{booktabs} 
\usepackage{pifont}   

\newcommand{\cmark}{\ding{51}}%
\newcommand{\xmark}{\ding{55}}%
\hypersetup{hidelinks,breaklinks=true}
\DeclareFontShape{T1}{ptm}{m}{scit}{<->ssub*ptm/m/it}{}
\usepackage{tabularx}      
\usepackage{ragged2e}      
\newcolumntype{L}{>{\RaggedRight\arraybackslash}X} 

\usepackage{tikz} 

\usetikzlibrary{shapes.geometric, arrows.meta, positioning, calc, backgrounds, fit}

\usepackage{mathtools} 
\usepackage{amssymb}   
\usepackage{amsthm}    
\usepackage{wrapfig}
\usepackage{bm}        
\usepackage{mathrsfs}  
\usepackage{dsfont}    

\usepackage{mleftright} 
\usepackage{microtype}  

\usepackage{amsmath} 
\DeclareMathOperator*{\E}{\mathbb{E}} 
\usepackage{subcaption}

\DeclareCaptionLabelFormat{subfigwithparent}{Figure~\thefigure#2}
\title{\textit{Stochastic} MeanFlow Policies: One-Step Generative Control with Entropic Mirror Descent}

\author{%
  \bf
    Zeyuan Wang$^{1}$ \quad
    Da Li$^{2,3}$ \quad
    Yulin Chen$^{1}$ \quad
    Yuehu Gong$^{4}$ \\
    \bf
    Yanming Guo$^{1}$ \quad
    Liang Bai$^{1}$ \quad 
    Tianyuan Yu$^{1}$ \quad 
    Yanwei Fu$^{4,5}$ \\[2mm]
    $^{1}$Laboratory for Big Data and Decision, National University of Defense Technology, China \\
    $^{2}$Samsung AI Center Cambridge, UK \\
    $^{3}$Queen Mary University of London, UK \\
    $^{4}$Fudan University, China \\
    $^{5}$Shanghai Innovation Institute, China
}

\makeatletter
\let\SMFPnatHyper\NAT@hyper@
\def\NAT@hyper@#1{\mbox{\SMFPnatHyper{#1}}}
\makeatother
\microtypesetup{protrusion=false}

\iclrfinalcopy

\begin{document}

\maketitle \fancyhead{} \renewcommand{\headrulewidth}{0pt}

\begin{abstract}

Online off-policy reinforcement learning (RL) is governed by two coupled design dimensions: policy representation (e.g. Gaussian vs generative policies) and optimisation methodology (e.g. soft actor-critic (SAC) vs mirror descent (MD)). Gaussian policies offer fast inference and tractable entropy estimation but have limited ability to model multimodal action distributions, whereas generative policies provide richer action distributions at the cost of iterative sampling or intractable entropy estimation. On the optimisation side, SAC-style entropic exploration and MD updates can be interpreted as minimising distinct Kullback-Leibler divergences: SAC-style exploration performs soft policy improvement towards a value-induced Boltzmann distribution, whereas MD constrains updates to remain close to the previous policy. Combining SAC-style entropy and MD regularisation supports exploration and conservative policy updates, but can yield multimodal targets beyond the representational capacity of unimodal Gaussian policies.  

This mismatch calls for an expressive policy with efficient sampling and explicit entropy control. To this end, we introduce Stochastic MeanFlow Policies (SMFP), a one-step generative policy class that combines MeanFlow-based noise-to-action mappings with Gaussian reparameterisation. The latent mapping represents diverse action modes, while the learnable conditional Gaussian scale yields an analytic entropy lower bound. We use this bound to regulate stochasticity alongside direct value guidance and advantage-weighted MeanFlow regression, making entropy-regularised mirror-descent learning practical with one-step generative policies. Empirically, on seven MuJoCo benchmarks, SMFP achieves strong performance against both Gaussian and generative baselines while maintaining single-step inference efficiency. Complementary experiments demonstrate multimodal action modelling and effective exploration in sparse-reward navigation.

\end{abstract}

\section{Introduction}
Online off-policy reinforcement learning (RL) has become a dominant paradigm for solving high-dimensional continuous control tasks, ranging from robotic manipulation to autonomous navigation~\citep{levine2018reinforcement, gu2017deep,margolis2023walk, tang2025deep}. Its success is largely governed by two key design axes: policy representation (\emph{how actions are modelled}) and optimisation methodology (\emph{how policies are updated}).

{Policy representation involves a trade-off between tractability and expressivity. Standard actor-critic methods~\citep{fujimoto2018addressing, haarnoja2018soft} favour tractable policy parameterisations. Gaussian policies provide efficient likelihood and entropy evaluation and fast inference, but have limited flexibility in representing separated action modes.}
In contrast, diffusion and flow-based policies~\citep{hansen2023idql,team2024octo,ding2024diffusion,dime,MaxEntDP,zhang2025sac} offer substantially greater expressivity and can model complex multimodal behaviours. However, these benefits come at the cost of intractable likelihoods and entropy estimation, as well as high inference latency due to iterative sampling~\citep{ho2020denoising,rectifiedflow}. {One-step MeanFlow policies~\citep{meanflowql} reduce sampling overhead, but their implicit action densities make explicit entropy control challenging.}

\begin{figure}[t]
    \centering
    \captionsetup[subfigure]{labelformat=simple,labelsep=space,skip=3pt}
    \begin{subfigure}[t]{0.485\textwidth}
        \centering
        \includegraphics[width=\linewidth,height=0.515\linewidth,keepaspectratio]{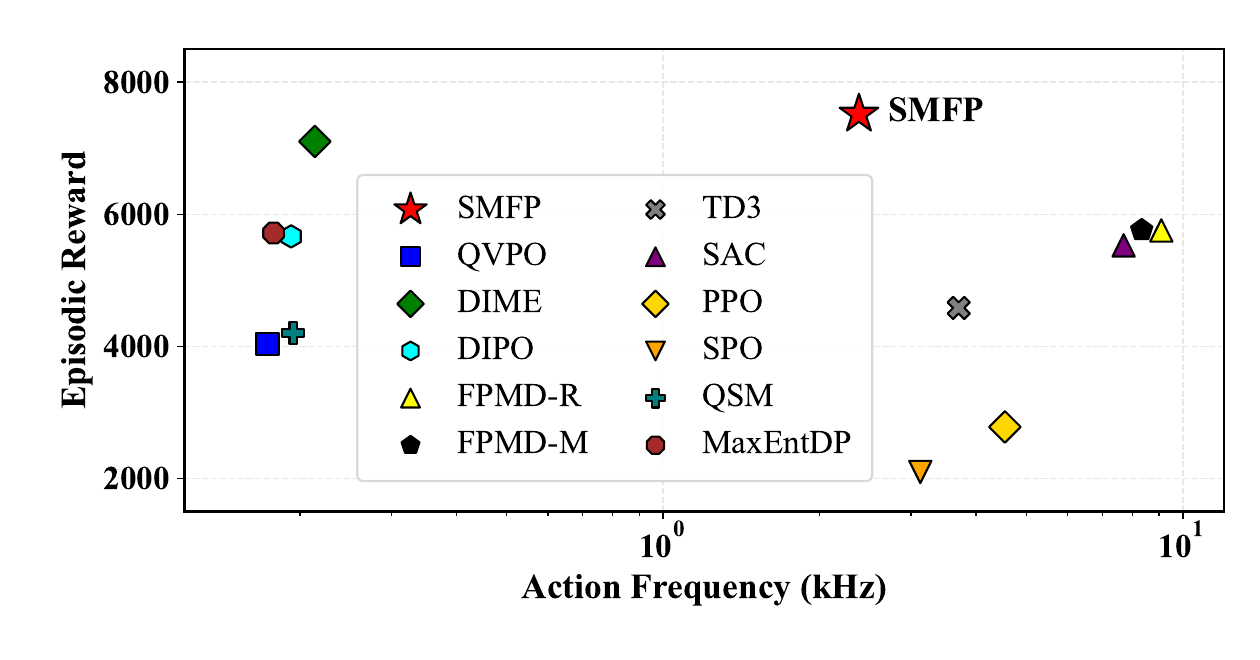}
        \caption{}
        \label{fig:introduction_teaser}
    \end{subfigure}\hfill
    \begin{subfigure}[t]{0.485\textwidth}
        \centering
        \includegraphics[width=\linewidth,height=0.515\linewidth,keepaspectratio]{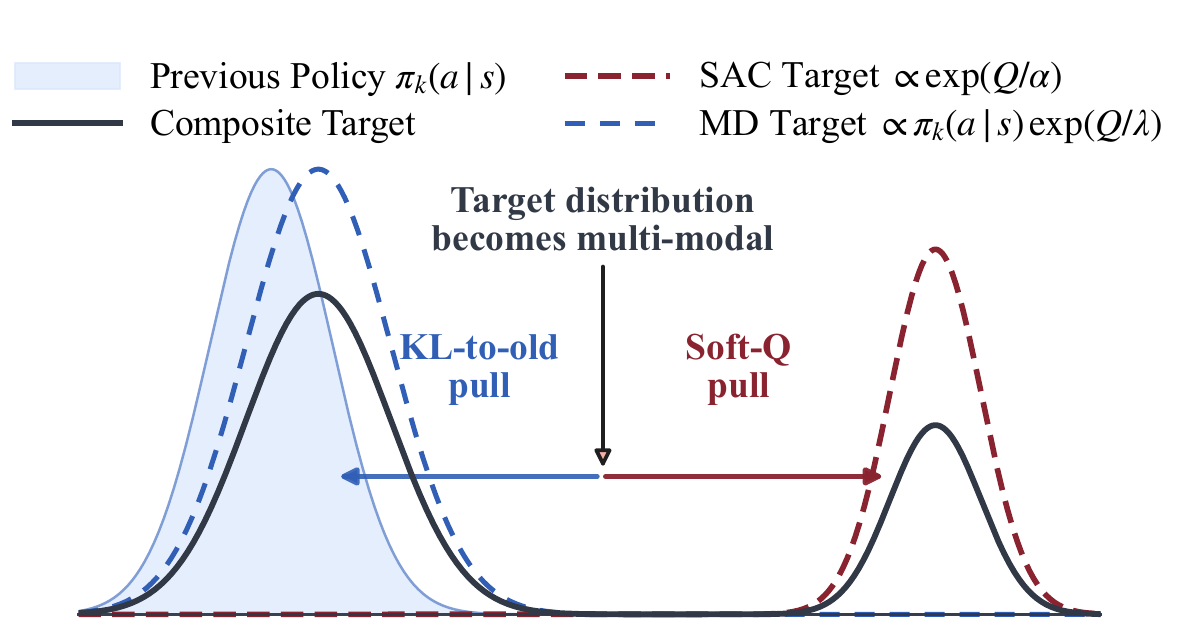}
        \caption{}
        \label{fig:introduction_tension}
    \end{subfigure}
    \caption{{(a) Ant-v4 episodic reward versus action frequency on an NVIDIA RTX 5090. (b) Schematic of the composite SAC-MD target, which balances value improvement with proximity to the previous policy. When separated high-value action regions have support under that policy, the target can retain multiple modes, motivating expressive action modelling.}}
    \label{fig:introduction_overview}
\end{figure}
{
Along another axis of optimisation methodology, a similar tension arises between exploration-oriented entropy maximisation, as in SAC~\citep{haarnoja2018soft}, and stability-focused trust-region methods such as mirror descent~\citep{mirrorwang2, mirror1}. Entropy-regularised approaches maximise a soft value function augmented with an entropy term, promoting broad exploration and mitigating premature convergence. However, they typically rely on soft maximisation and lack explicit control over policy updates, which can lead to instability during rapid learning. {In contrast, mirror-descent-based methods~\citep{beck2003mirror,raskutti2015information} constrain or penalise KL divergence from the previous policy, encouraging conservative updates and mitigating abrupt policy changes.}
Fundamentally, {these two optimisation approaches} can be interpreted as minimising distinct Kullback-Leibler divergences~\citep{mirror1}, capturing complementary aspects of policy optimisation: exploration vs stability.

\emph{Can SAC-style exploration be combined with the stability of mirror-descent updates?}
{These objectives are complementary: value maximisation favours high-value actions, while mirror descent regularises deviations from the previous policy. Their composite target can retain multiple modes when separated high-value action regions have support under that policy, as illustrated schematically in Figure~\ref{fig:introduction_tension}.} However, standard Gaussian policies are inherently unimodal and therefore fail to represent such targets. This mismatch between optimisation objectives and policy representation raises a fundamental question: \emph{can expressive generative policies bridge this gap?}
}

SAC-style entropy regularisation requires a tractable estimate of policy entropy. However, for expressive generative policies, exact likelihoods and entropies are generally intractable and often rely on costly estimators or trajectory-level approximations~\citep{hutchinson1989stochastic,dime,zhang2025reinflow}. In contrast, diagonal Gaussian parameterisations~\citep{kingma2013auto,rezende2014stochastic,yu2021simple} enable efficient stochastic sampling and closed-form entropy control, suggesting a principled route to bridge generative expressivity with entropy-driven optimisation.
{To this end, we introduce Stochastic MeanFlow Policies (SMFP), which equip a one-step MeanFlow mapping~\citep{meanflowql} with a learnable conditional Gaussian perturbation. The mapping represents diverse action modes, while the conditional scale provides an analytic entropy lower bound. Optimising this surrogate together with direct value gradients and advantage-weighted MeanFlow regression gives the policy explicit stochasticity control without iterative action generation.}

Empirically, SMFP achieves strong performance across seven MuJoCo continuous-control benchmarks, matching or outperforming iterative diffusion and flow baselines while retaining efficient one-step inference. {Prior work reports limited practical gains from mirror descent in conventional off-policy actor-critic methods~\citep{neumann2025investigating}. Our ablations show that entropy regularisation and mirror-descent regression each improve SMFP, supporting their joint use within the proposed policy.} {Complementary evaluations on Push-T and sparse-reward AntMaze further demonstrate multimodal action modelling and effective exploration, respectively (Figs.~\ref{fig:intro_network}b and~\ref{fig:antmaze_success_curves}).}

\section{{Related Work}}
\textbf{Maximum Entropy Reinforcement Learning}.
Maximum entropy (MaxEnt) RL, standardised by algorithms like SAC~\citep{haarnoja2018soft}, balances exploration and exploitation by augmenting the reward objective with entropy regularisation. {CrossQ~\citep{crossq} improves sample efficiency through batch renormalisation~\citep{ioffe2017batchrenorm}, while BRO~\citep{nauman2024bro} combines larger critics with strong regularisation and optimistic exploration. Both retain Gaussian policies with limited flexibility in modelling separated action modes~\citep{chi2023diffusion,pan2025much}.} {LSAC~\citep{ishfaq2025lsac} promotes exploration through uncertainty-aware critic learning. Expressive MaxEnt policies include SVGD-based samplers~\citep{haarnoja2017reinforcement} and generative models, including} diffusion-based algorithms like DIME~\citep{dime}, DACER~\citep{wang2024diffusion}, and MaxEntDP~\citep{MaxEntDP} that approximate score functions or entropy surrogates, alongside flow-based methods such as ReinFlow~\citep{zhang2025reinflow} and SAC Flow~\citep{zhang2025sac} that leverage velocity reparameterisation or noise injection for tractable entropy maximisation. {Recent FLAC~\citep{lv2026flac} avoids action-density estimation through kinetic-energy-regularised bridge matching, but regularises paths rather than action entropy.}

\textbf{Mirror Descent Policy Optimisation}.
{Trust-region and proximal methods, notably TRPO~\citep{schulman2015trust} and PPO~\citep{schulman2017PPO}, seek to stabilise learning by limiting departures from the current policy.} {Mirror Descent Policy Optimisation (MDPO)~\citep{mirror1, neu2017unified, vieillard2020leverage} formalises this principle through policy improvement regularised by a Bregman divergence.} While recent research aims to instantiate this framework using highly expressive generative models, such as Diffusion~\citep{ho2020denoising, song2020score} and Flow Matching~\citep{rectifiedflow}, the direct application of the theoretical MDPO update is impeded by the intractability of evaluating exact likelihoods for these priors. {Consequently, prior work uses tractable surrogates to bypass exact density evaluation: DIPO~\citep{yang2023dipo} improves sampled actions using critic gradients and trains a diffusion policy to imitate them, QVPO~\citep{ding2024diffusion} fits diffusion policies with a Q-weighted variational objective, and FlowRL~\citep{lv2025flow} uses value-weighted flow matching with Wasserstein-2 regularisation.}

\textbf{One-Step Generative Policy}.
Recent advancements have sought to bypass the computational burden of iterative sampling by adopting one-step generative policies, including distillation-based methods~\citep{fql}, Consistency Policies~\citep{ding2023consistency}, Shortcut Models~\citep{fql_shortcut}, and MeanFlow formulations~\citep{chen2025one,meanflowql}. {Although latent noise allows these mappings to generate stochastic actions, the induced marginal density is generally implicit. Efficient one-step sampling therefore does not by itself provide a tractable mechanism for explicit entropy control in online RL.} 


\textbf{More Discussion}. Unlike prior one-step flow or MeanFlow policies that primarily focus on deterministic fast action generation or mirror-descent projection~\cite{chen2025one,meanflowql,fql_shortcut}, SMFP explicitly constructs a stochastic one-step generative policy whose conditional variance admits a tractable entropy lower bound. This makes the policy compatible with SAC-style entropy regularisation while retaining the stability of off-policy mirror descent. Thus, the key contribution is not merely replacing diffusion sampling with MeanFlow, but enabling a tractable objective in which entropy-driven exploration, value-weighted mirror descent, and one-step generative inference are jointly optimised.

\section{Preliminaries}
\label{sec:preliminaries}
Online reinforcement learning is modelled as an infinite-horizon Markov Decision Process (MDP) $\mathcal{M}=(\mathcal{S},\mathcal{A},p,r,\gamma)$, where $\mathcal{S}$ is the state space, $\mathcal{A}\subseteq\mathbb{R}^d$ is the continuous action space, $p$ is the transition density, $r$ is a bounded reward function, and $\gamma\in[0,1)$ is the discount factor. We use $h$ to index the discrete RL time step, so the dynamics and rewards are given by $p(s_{h+1}|s_h,a_h)$ and $r(s_h,a_h)$, respectively. In the online off-policy setting, the agent collects experience through environment interaction and stores transitions in a replay buffer $\mathcal{D}$ for off-policy updates. We consider two policy optimisation frameworks: Maximum Entropy RL and Policy Mirror Descent. For notation, the policy and Q-network are parameterised by $\theta$ and $\phi$, respectively; $\alpha$ denotes the SAC entropy coefficient, and $\lambda$ denotes the Policy Mirror Descent regularisation coefficient.

\subsection{Maximum Entropy Reinforcement Learning}
\label{sec:sac_prelim}
Maximum Entropy RL augments the return with an entropy bonus to encourage
exploration. SAC~\citep{haarnoja2018soft} instantiates this objective in an
off-policy actor-critic framework by maximising
$\E_{\pi}\!\left[\sum_{h=0}^{\infty}\gamma^h
\bigl(r(s_h,a_h)+\alpha\mathcal{H}(\pi(\cdot|s_h))\bigr)\right]$,
{where $\alpha$ sets the entropy weight.}

\textbf{Soft Policy Evaluation.}
{We evaluate $Q_\phi(s,a)$ using the soft Bellman operator}
\begin{equation}
    \mathcal{T}^\pi Q(s_h,a_h)
    \triangleq
    r(s_h,a_h)
    +
    \gamma
    \mathbb{E}_{s_{h+1}\sim p(\cdot|s_h,a_h)}
    \left[
        V^\pi(s_{h+1})
    \right],
    \label{eq:pre_critic_loss}
\end{equation}
where
$V^\pi(s)=\mathbb{E}_{a\sim\pi(\cdot|s)}
\left[Q_\phi(s,a)-\alpha\log\pi(a|s)\right]$.
The critic parameters are updated by minimising the squared Bellman residual.

\textbf{Soft Policy Improvement.}
The policy improvement step can be written as projecting the parametric policy
onto the Boltzmann distribution induced by the soft Q-function:
\begin{equation}
    \mathcal{L}_\pi(\theta)
    =
    \mathbb{E}_{s_h \sim \mathcal{D}}
    \left[
        D_{\text{KL}}
        \left(
            \pi_\theta(\cdot|s_h)
            \bigg\|
            \frac{
                \exp\!\left(\frac{1}{\alpha}Q_\phi(s_h,\cdot)\right)
            }{
                Z(s_h)
            }
        \right)
    \right].
    \label{eq:sac_kl}
\end{equation}
Using the reparameterisation trick, this yields the standard tractable SAC actor
loss
\begin{equation}
    \mathcal{L}_\pi(\theta)
    =
    \mathbb{E}_{\substack{s_h\sim\mathcal{D}\\ a_h\sim\pi_\theta}}
    \left[
        \alpha \log \pi_\theta(a_h|s_h)
        -
        Q_\phi(s_h,a_h)
    \right],
    \label{eq:sac_loss}
\end{equation}
where $a_h$ is sampled from the reparameterised policy. In standard SAC,
$\pi_\theta$ is typically parameterised as a squashed Gaussian, which enables
{efficient sampling and likelihood evaluation, with limited flexibility in modelling separated action modes.}

\subsection{Policy Mirror Descent}
\label{sec:mdpo_prelim}

Mirror Descent (MD) provides a general framework for regularised policy optimisation in reinforcement learning~\citep{schulman2015trust, mirror1}. Rather than prescribing a specific objective, MD constrains each policy update to remain close to the previous policy while optimising a chosen policy objective. In this work, we consider a value-based instantiation, where the policy is improved with respect to the action-value estimate. {Given $\pi_k$, we minimise}
\begin{equation}
    \mathcal{L}_{\text{MD}}(\pi) = \E_{s \sim \mathcal{D}, a \sim \pi} \left[ -Q^{\pi_k}(s, a)  + \lambda D_{\text{KL}}(\pi(\cdot|s) \,\|\, \pi_k(\cdot|s))\right].
    \label{eq:pmd_objective}
\end{equation}
For this value-based objective, the update admits an analytical solution, where the optimal policy is proportional to an energy-based reweighting of the prior policy $\pi_k$:
\begin{equation}
    \pi_{k+1}(a|s) = \frac{\pi_k(a|s) \exp\left( \frac{1}{\lambda} Q^{\pi_k}(s, a) \right)}{Z(s)},
    \label{eq:pmd_closed_form}
\end{equation}
where $Z(s) = \int \pi_k(a'|s) \exp( \frac{1}{\lambda} Q^{\pi_k}(s, a') ) \, da'$ is the partition function. This multiplicative update shifts probability mass towards high-value actions while keeping the policy close to the previous iterate $\pi_k$, thereby improving update stability.

\subsection{Flow Matching and MeanFlow Variants} 
\label{sec:flow-matching}
Flow Matching learns to transform noise into data via time-dependent velocities, whereas MeanFlow enables one-step sampling by modelling the time-averaged velocity.

\textbf{Flow Matching.} Flow matching~\citep{lipman2023flow,rectifiedflow} is a generative modelling framework that learns a continuous time-dependent velocity field bridging a simple base distribution and a complex data distribution. In our context, let $e \sim p_{\text{prior}}$ be sampled from a simple prior (e.g. a standard Gaussian) and $a \sim p_{\text      {data}}$ from the empirical action data distribution. We consider an interpolation $a_t = (1-t)\,a + t\,e$ for $t \in [0,1]$. Flow matching trains a velocity field $v_\theta(a_t,t)$ to predict the ground-truth displacement between $a$ and $e$. Specifically, one trains $v_\theta$ by {minimising}:
\begin{equation}
\min_{\theta}\; \underset{\substack{a \sim p_{\text{data}},\\ \,e \sim p_{\text{prior}},\,t\sim \mathcal{U}(0,1)}}{\mathbb{E}}\big[\|v_\theta(a_t,t) - (e - a)\|_2^2\big],
\label{eq:flow_matching}
\footnote{{State $s$ is also a model input; for compact notation, we denote $(s,a_t)$ simply by $a_t$.}}
\end{equation}
where $a_t = (1-t)a + t e$ as above. By learning to predict $e - a$, the model captures the time-dependent instantaneous velocity needed to transform $a$ into $e$ (or vice versa). 

\textbf{MeanFlow and Variants.} MeanFlow~\citep{meanflow} builds on flow matching by introducing an \emph{average velocity} field $u(a_t, b, t)$. For any $0 \leq b < t \leq 1$, the average velocity is defined as the time-average of the instantaneous velocity $v$ over the interval $[b,t]$:
\begin{equation}
u(a_t, b, t) \;\coloneqq\; \frac{1}{\,t - b\,} \int_{b}^{t} v(a_\tau,\tau)\,d\tau.
\label{eq:meanflow_identity}
\end{equation}
Differentiating both sides with respect to $t$ (treating $b$ as constant) yields the \emph{MeanFlow Identity} ~\citep{meanflow}:
\begin{equation}
u(a_t,b,t) \;=\; v(a_t,t)\;-\;(t-b)\,\frac{d}{d t}u(a_t,b,t).
\label{eq:meanflow_diff}
\end{equation}

MeanFlow establishes a principled relationship between the average velocity $u$ and the instantaneous velocity $v$. Crucially, it enables \textit{one-step} generation: by choosing $b=0$ and $t=1$ in \eqref{eq:meanflow_identity}, one obtains the one-step sampler $\hat{a}=e-u_\theta(e,b=0,t=1)$. This formula generates $\hat{a}$ through a single velocity estimate from the learned $u_\theta$, without time integration.

{For Q-learning, residual MeanFlow~\citep{meanflowql} directly parameterises the mapping from latent noise to actions:} 
\begin{equation}
    f_\theta(a_t, b, t) = a_t - u_\theta(a_t, b, t).
    \label{equ:meanflowql_form}
\end{equation}
{This data-prediction parameterisation~\citep{li2025back} outputs an action in one forward pass. Its differentiable noise-to-action mapping also allows direct optimisation through the critic.}

\section{Stochastic MeanFlow Policy Framework}
\begin{figure}[t]
    \vspace{-0.5cm}
  \centering
  \includegraphics[trim=0 110 60 0, clip, width=0.95\textwidth]{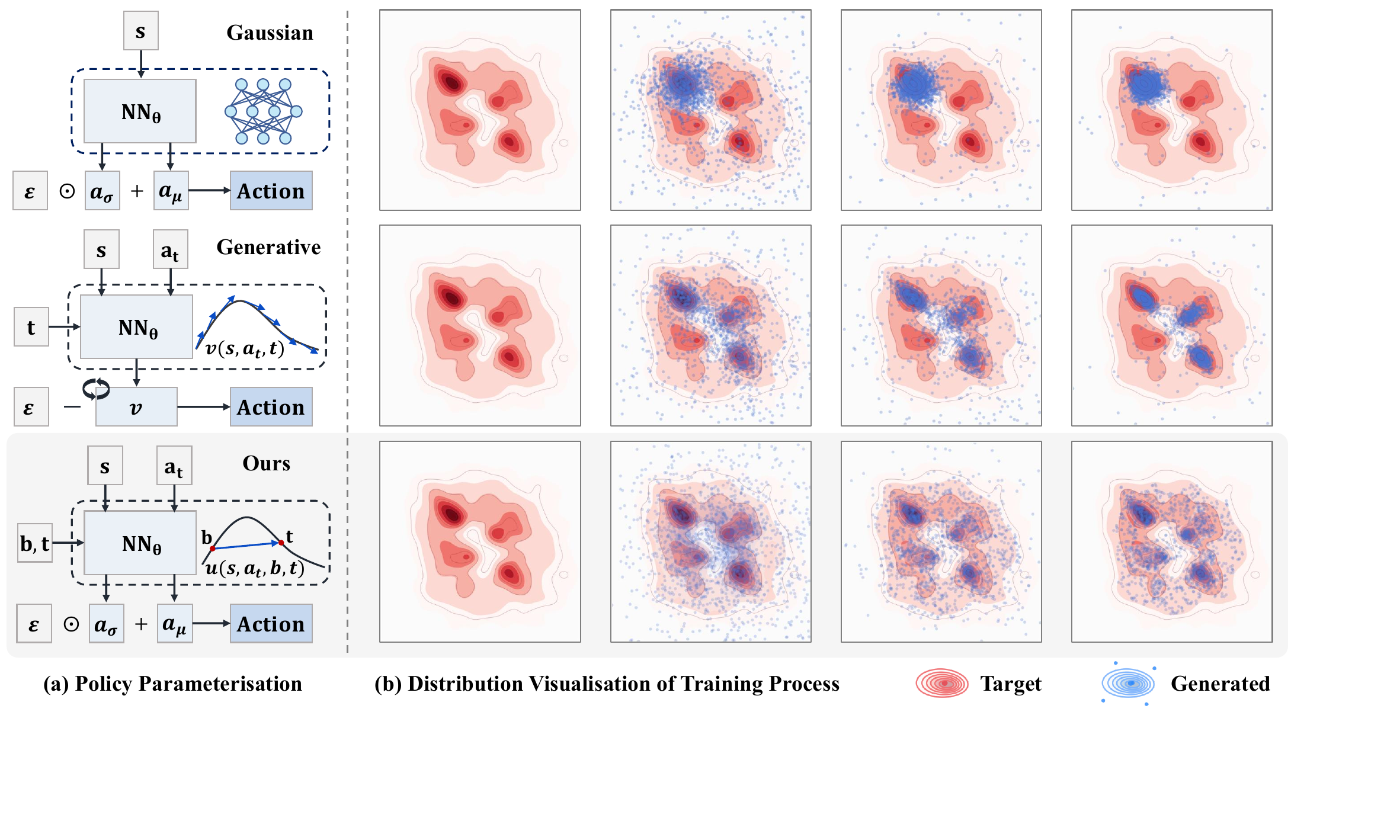}
  \vspace{-0.2cm}
\caption{
Comparison of policy parameterisations and action-distribution evolution on Push-T.
{(a) Gaussian, conventional flow-based, and SMFP policies. Conventional flow policies estimate instantaneous velocities $v(s,a_t,t)$; SMFP uses the MeanFlow relation over $[b,t]$ to learn a direct noise-to-action mapping, with Gaussian reparameterisation providing a tractable entropy lower bound.}
{(b) Training evolution of generated action distributions on Push-T. Red contours indicate target modes and blue samples indicate generated actions. SMFP maintains samples around the separated target modes across the displayed stages, illustrating multimodal action modelling. The complementary AntMaze evaluation in Fig.~\ref{fig:antmaze_success_curves} examines exploration and sustained goal-reaching performance.}
}
\label{fig:intro_network}
\vspace{-0.5cm}
\end{figure}

{SMFP combines one-step generative action modelling with explicit stochasticity control and off-policy improvement (Fig.~\ref{fig:intro_network}a). A residual MeanFlow mapping represents diverse action regions, while a learnable conditional Gaussian perturbation controls dispersion around its outputs. We first define this policy, then derive a tractable entropy surrogate and combine it with value-weighted MeanFlow regression for off-policy learning.}

\subsection{Stochastic MeanFlow Policy Architecture}
\label{sec:smf_architecture}
{We use the residual MeanFlow formulation~\citep{meanflowql} as the policy backbone. It implements data prediction ($x$-prediction)~\citep{li2025back}, enabling one-step generation and direct Q-gradient propagation without differentiating through an ODE solver.}  

\textbf{Stochastic Reparameterisation.}
{The MeanFlow mapping generates diverse actions from latent noise, while Gaussian reparameterisation equips each latent-conditioned output with an explicit density and a learnable noise scale. Figure~\ref{fig:intro_network}a illustrates how this mechanism extends Eq.~\ref{equ:meanflowql_form}:}
\begin{equation}
g_\theta(a_t, b, t) = \underbrace{a_t - u_\theta(a_t, b, t)}_{Action_\mu} + \underbrace{\sigma_\theta(a_t, b, t)}_{Action_\sigma} \odot \epsilon,
\label{eq:reparam_action}
\end{equation}
{where $g_\theta$ denotes the generative policy network and $a_t = (1-t)a + t e$ is the interpolated action. Here, $e$ and $\epsilon$ are independent standard Gaussian noise variables.} 
The two noise variables play different roles: $e$ determines the latent generative branch of the one-step MeanFlow mapping, while $\epsilon$ parameterises the conditional stochasticity around that branch and is the source of the tractable entropy surrogate.
%
%
%
{Here, $u_\theta$ defines the residual mean $a_t-u_\theta$, $\sigma_\theta$ predicts the conditional standard deviation, and $\odot$ denotes element-wise multiplication.} 
The one-step generative policy is then obtained by setting $b=0$, $t=1$, and $a_1=e$, yielding the predicted action $\hat{a} = g_\theta(a_t = e, b=0, t=1)$.

\textbf{Optimisation Objective.}
{We combine the SAC objective (Eq.~\ref{eq:sac_loss}) with a KL penalty relative to the previous policy $\pi_{\theta_{\mathrm{old}}}$ (Eq.~\ref{eq:pmd_objective}). Entropy encourages exploration, whereas the KL penalty moderates changes relative to the previous policy. At iteration $k$, the combined objective is:}
\begin{equation}
\begin{split}
    \mathcal{L}(\theta) = \E_{\substack{s \sim \mathcal{D}}} \Big[ & -Q(s, g_\theta(s)) + \alpha \log \pi_\theta(\cdot|s) + \lambda D_{\text{KL}}(\pi_\theta(\cdot|s) \| \pi_{\theta_{\text{old}}}(\cdot|s)) \Big],
\end{split}
\label{eq:overall}
\end{equation}
where $\alpha$ weights entropy regularisation, and $\lambda$ parameterises the inverse mirror descent step size. Here, $g_\theta$ serves as the reparameterised sampler for action generation, while $\pi_\theta$ denotes the induced density for computing the entropy and KL terms. 
This formulation presents two computational challenges specific to implicit generative policies: (1) calculating the log-likelihood $\log \pi_\theta(a|s)$ for a flow-based distribution without explicit density, and (2) computing the KL divergence between two stochastic MeanFlow policies. To address these, we derive a tractable entropy surrogate in Sec.~\ref{sec:sac_method} and formulate the mirror descent training objective in Sec.~\ref{sec:md_method}.

\subsection{Tractable Entropy Regularisation}
\label{sec:sac_method}
{For generative policies, SAC-style entropy regularisation is challenging because the marginal density is implicit. Density tracking and trajectory-level approximations can add substantial overhead to online RL~\citep{hutchinson1989stochastic,grathwohlffjord,dime,zhang2025reinflow}. SMFP instead uses its explicit conditional Gaussian policy to control stochasticity. Conditioned on $e\sim\mathcal{N}(0,I)$, the action has a tractable density, while the marginal policy $\pi_\theta(a|s)=\int\pi_\theta(a|s,e)p(e)\,\mathrm{d}e$ retains the expressivity of a continuous mixture. Proposition~\ref{prop:entropy_lb} establishes a conditional-entropy lower bound on the marginal entropy. Dropping the additive Gaussian constant gives}
\begin{equation}
    \widetilde{\mathcal{H}}(\pi_\theta \mid e)
    =
    \mathbb{E}_{e}
    \left[
        \sum_{i=1}^{d}
        \log \sigma_\theta^{(i)}(a_t,b,t)
    \right].
    \label{eq:entropy_calc}
\end{equation}
{This lower bound makes the expected log-standard deviation a tractable measure of conditional dispersion. It complements the variation induced by the latent mapping, giving the expressive policy an explicit mechanism for controlling stochasticity. Proposition~\ref{prop:kl_limit} provides a complementary diffuse-prior interpretation. In practice, we encourage a minimum noise scale through a hinge penalty, which activates when the normalised log-standard deviation falls below a threshold:}
\begin{equation}
    \mathcal{L}_{\mathrm{Ent}}(\theta)
    =
    \mathbb{E}
    \left[
        \operatorname{ReLU}
        \left(
            \kappa_{\sigma}
            -
            \frac{1}{d}
            \sum_{i=1}^{d}
            \log \sigma_\theta^{(i)}(a_t,b,t)
        \right)
    \right].
    \label{eq:entropy_term}
\end{equation}
{Here, $d$ is the action dimension and $\kappa_{\sigma}$ is the threshold for the average log-scale. The penalty encourages dispersion below this threshold and becomes inactive above it, where the remaining actor terms continue to shape the scale. Appendix~\ref{app:entropy_proofs} relates this regulariser to the entropy bound; Figs.~\ref{fig:ablation_ent}--\ref{fig:noise_scale} evaluate its effect on learning.}

\subsection{Optimisation via Off-Policy Mirror Descent}
\label{sec:md_method}
{For the value-based mirror-descent component, Eq.~\eqref{eq:pmd_closed_form} gives the energy-based target $\pi^*(\cdot|s) \propto \pi_{\theta_{\mathrm{old}}}(\cdot|s) \exp(Q(s, \cdot)/\lambda)$. This target motivates weighted regression towards high-value actions under the previous policy, alongside the entropy control in Sec.~\ref{sec:sac_method}.}

Since $\pi^*$ is a non-parametric and unnormalised target, we do not sample from it directly. Instead, following weighted policy projection~\citep{abdolmaleki2018maximum, chen2025one}, we use samples from the previous policy as proposals and weight them according to their relative target density.
\begin{equation}
    \omega(s,a)
    \propto
    \exp\!\left(\frac{Q(s,a)}{\lambda}\right)
    \propto
    \exp\!\left(\frac{A^{\pi_{\theta_{\mathrm{old}}}}(s,a)}{\lambda}\right),
    \label{eq:adv_weight}
\end{equation}
{where the second proportionality holds because the state-value baseline is action-independent.}  

For SMFP, however, the marginal likelihood required by this projection is intractable. We therefore retain the weighting induced by the mirror-descent target, but replace the marginal log-likelihood fitting term with the tractable stochastic MeanFlow regression surrogate:
\begin{equation}
    \mathcal{L}_{\mathrm{SMF}}(\theta; a, s)
    =
    \mathcal{L}_{\delta}
    \left(
    g_\theta(a_t,b,t)
    -
    \operatorname{sg}(g_{\mathrm{tgt}})
    \right),
    \label{eq:smf_loss}
\end{equation}
where $\mathcal{L}_{\delta}$ is the Huber loss, and
$\operatorname{sg}(\cdot)$ denotes the stop-gradient operator. Following the
MeanFlow identity, the target is
\begin{equation}
\begin{split}
    g_{\mathrm{tgt}}
    \triangleq\,
    a_t
    + (t-b-1)v(a_t,t)
    + \sigma_\theta \odot \epsilon
    - (t-b)
    \left[
        \texttt{jvp}(g_\theta, \cdot)
        -
        \texttt{jvp}(\sigma_\theta, \cdot)\odot\epsilon
    \right],
\end{split}
\label{eq:g_tgt_def}
\end{equation}
with the derivation provided in Appendix~\ref{app:derivation_smf}.
{In practice, noisy online value estimates can make the exponential weights in Eq.~\eqref{eq:adv_weight} unstable. Following QVPO~\citep{ding2024diffusion}, we use the positive part of the estimated advantage:}
\begin{equation}
    \omega_{\mathrm{adv}}(s,a)
    =
    \max\bigl(0, Q(s,a)-V(s)\bigr),
    \label{equ:adv}
\end{equation}
{with the rationale given in Appendix~\ref{app:addvantage}.} This also avoids a
separate value network: {$V(s)$ is estimated by averaging the Q-values of $N_{\mathrm{adv}}$ actions drawn directly from the pre-update policy.} The resulting
mirror-descent regression objective is
\begin{equation}
    \mathcal{L}_{\mathrm{MD}}(\theta)
    =
    \mathbb{E}_{a\sim \pi_{\theta_{\mathrm{old}}}}
    \left[
        \omega_{\mathrm{adv}}(s,a)
        \mathcal{L}_{\mathrm{SMF}}(\theta; a, s)
    \right].
    \label{eq:final_loss_md}
\end{equation}

{We combine these components into a tractable actor objective: the direct Q-gradient improves generated actions, the conditional entropy-floor loss maintains stochasticity, and the advantage-weighted MeanFlow loss fits high-value policy samples. Together, these terms implement the value-improvement, entropy-control, and policy-regularisation roles motivated by Eq.~\eqref{eq:overall}:}
\begin{equation}
    \mathcal{L}(\theta)
    =
    \mathbb{E}_{s\sim\mathcal{D}}
    \left[
        -Q(s,g_\theta(\cdot))
        + \alpha \mathcal{L}_{\mathrm{Ent}}(\theta)
        + \lambda \mathcal{L}_{\mathrm{MD}}(\theta)
    \right].
    \label{eq:overall_loss_tractable}
\end{equation}

{Algorithm~\ref{alg:SMFP} summarises training. For interaction and critic targets, value-guided (best-of-$K$) selection chooses the highest-Q action from $K_b$ and $K_t$ candidates, respectively, with one policy evaluation per candidate. Implementation details appear in Appendix~\ref{app:algorithm}.}

\begin{algorithm}[!htbp]
\caption{Stochastic MeanFlow Policy Optimisation}
\label{alg:SMFP}
{
\footnotesize
\begin{algorithmic}
\State Initialise replay buffer $\mathcal{D}$, policy $g_\theta$, critic $Q_\phi$, and target critic $Q_{\bar\phi}$.
\For{$i=1$ \textbf{to} $N$}
    \State Collect transitions with value-guided actions and store them in $\mathcal{D}$.
    \State Sample a mini-batch $(s,a,r,s')$ from $\mathcal{D}$ and set $\theta_{\mathrm{old}}\gets\theta$.
    \BeginBox[fill=mask!60]
    \LComment{\color{hiccup} Train Stochastic One-Step MeanFlow Policy $\pi_\theta$}
    \State Draw $N_{\mathrm{adv}}$ actions directly from $\pi_{\theta_{\mathrm{old}}}(\cdot\mid s)$; estimate $V(s)$ and weights via Eq.~\eqref{equ:adv}.
    \State Construct targets (Eq.~\eqref{eq:g_tgt_def}) and update {\color{hiccup}$\theta$} by minimising Eq.~\eqref{eq:overall_loss_tractable}. \Comment{SMFP Loss}
    \EndBox
    \BeginBox[fill=gray!8]
    \LComment{\color{hiccup} Train Critic $Q_\phi$}
    \State Select value-guided next actions and update $\phi$ via Eq.~\eqref{eq:critic_loss}. \Comment{Critic Loss}
    \State Update the target critic $\bar\phi$ by exponential moving average.
    \EndBox
\EndFor
\State \Return One-step SMFP policy $\pi_\theta$.
\end{algorithmic}
}
\end{algorithm}

\section{Experiments}
\label{sec:experiments}
{We evaluate SMFP on MuJoCo benchmarks following the general experimental setup. We further visualise its multimodal action distributions on Push-T~\citep{chi2023diffusion} and assess its exploration capabilities on the sparse-reward AntMaze tasks introduced in DDiffPG~\citep{li2024learning}.}

\begin{table*}[t]
\centering
\vspace{-0.2cm}
\caption{Performance comparison of SMFP with 15 other online RL algorithms across 20 random seeds. The corresponding learning curves are shown in Fig.~\ref{fig:exp_result_curve}.}
\label{tab:main_results}

\definecolor{rank1}{RGB}{179, 217, 255}
\definecolor{rank2}{RGB}{217, 236, 255}
\definecolor{rank3}{RGB}{240, 248, 255}

\newcommand{\first}[1]{\cellcolor{rank1}\textbf{#1}}
\newcommand{\second}[1]{\cellcolor{rank2}#1}
\newcommand{\third}[1]{\cellcolor{rank3}#1}

\scalebox{0.58}{
\begin{threeparttable}
\begin{tabular}{llcccccccc}
\toprule
\multicolumn{2}{c}{{Method}} & {NFE} & \textsc{Hopper} & \textsc{Walker2d} & \textsc{Ant} & \textsc{HalfCheetah} & \textsc{Humanoid} & \textsc{HumanoidStandup} & \textsc{Swimmer} \\ 
\midrule
\multirow{4}{*}{\begin{tabular}[c]{@{}l@{}}\textbf{Gaussian}\\ \textbf{Policy RL}\end{tabular}} 
& PPO & 1 & 3154.3 {\scriptsize $\pm$ 426.2} & 3751.5 {\scriptsize $\pm$ 609.1} & 2781.9 {\scriptsize $\pm$ 74.1} & 4773.5 {\scriptsize $\pm$ 53.4} & 713.7 {\scriptsize $\pm$ 85.9} & 56368.4 {\scriptsize $\pm$ 563.2} & 86.2 {\scriptsize $\pm$ 12.3} \\
& SPO & 1 & 2212.8 {\scriptsize $\pm$ 988.4} & 3321.8 {\scriptsize $\pm$ 1328.9} & 2100.2 {\scriptsize $\pm$ 302.4} & 4008.2 {\scriptsize $\pm$ 246.8} & 797.4 {\scriptsize $\pm$ 262.1} & 62875.2 {\scriptsize $\pm$ 3521.8} & 58.4 {\scriptsize $\pm$ 14.2} \\
& TD3 & 1 & 3267.5 {\scriptsize $\pm$ 8.5} & 3513.9 {\scriptsize $\pm$ 40.7} & 4583.8 {\scriptsize $\pm$ 69.5} & 10388.6 {\scriptsize $\pm$ 80.4} & 5353.5 {\scriptsize $\pm$ 53.7} & 105243.8 {\scriptsize $\pm$ 637.1} & 80.4 {\scriptsize $\pm$ 11.2} \\
& SAC & 1 & 2996.6 {\scriptsize $\pm$ 111.9} & 4988.1 {\scriptsize $\pm$ 80.0} & 5530.6 {\scriptsize $\pm$ 1000.3} & 10616.9 {\scriptsize $\pm$ 72.8} & 5159.7 {\scriptsize $\pm$ 475.3} & 116073.6 {\scriptsize $\pm$ 1254.4} & 69.2 {\scriptsize $\pm$ 10.1} \\ 
\midrule
\multirow{6}{*}{\begin{tabular}[c]{@{}l@{}}\textbf{Diffusion}\\ \textbf{Policy RL}\end{tabular}} 
& DIPO & 20 & 1295.4 {\scriptsize $\pm$ 7.0} & 2181.7 {\scriptsize $\pm$ 1625.7} & 5665.9 {\scriptsize $\pm$ 14.7} & 9590.5 {\scriptsize $\pm$ 67.5} & 4945.5 {\scriptsize $\pm$ 898.6} & 105163 {\scriptsize $\pm$ 2162.3} & 52.2 {\scriptsize $\pm$ 7.3} \\
& QSM & 20 & 2154.7 {\scriptsize $\pm$ 998.2} & 3613.4 {\scriptsize $\pm$ 1443.5} & 4206.4 {\scriptsize $\pm$ 143.5} & 3888.2 {\scriptsize $\pm$ 632.6} & 4793.1 {\scriptsize $\pm$ 229.5} & 97120.1 {\scriptsize $\pm$ 927.1} & 61.6 {\scriptsize $\pm$ 7.3} \\
& { QVPO} & 20 & 3364.4 {\scriptsize $\pm$ 325.0} & 3826.3 {\scriptsize $\pm$ 1909.4} & 4040.1 {\scriptsize $\pm$ 1085.6} & 8759.4 {\scriptsize $\pm$ 696.3} & 5176.6 {\scriptsize $\pm$ 410.4} & 137754.7 {\scriptsize $\pm$ 80529.5} & \third{112.4 {\scriptsize $\pm$ 15.1}} \\
& DIME\tnote{*} & 16 & 2567.6 {\scriptsize $\pm$ 46.3} & \second{6447.8 {\scriptsize $\pm$ 104.6}} & \second{7103.6 {\scriptsize $\pm$ 85.2}} & \third{13464.1 {\scriptsize $\pm$ 139.2}} & \first{11814.1 {\scriptsize $\pm$ 164.1}} & 111336.9 {\scriptsize $\pm$ 744.5} & \second{118.8 {\scriptsize $\pm$ 4.5}} \\
& MaxEntDP & 20 & \second{3520.9 {\scriptsize $\pm$ 67.3}} & 5142.2 {\scriptsize $\pm$ 98.2} & 5717.9 {\scriptsize $\pm$ 134.5} & 11281.5 {\scriptsize $\pm$ 45.1} & 5749.6 {\scriptsize $\pm$ 105.6} & 158656.4 {\scriptsize $\pm$ 1432.8} & 90.3 {\scriptsize $\pm$ 22.7} \\
& DPMD & 20 & 3317.9 {\scriptsize $\pm$ 55.2} & 5023.2 {\scriptsize $\pm$ 96.4} & \third{5931.8 {\scriptsize $\pm$ 104.6}} & 12931.4 {\scriptsize $\pm$ 441.1} & 6835.4 {\scriptsize $\pm$ 532.8} & 186483.3 {\scriptsize $\pm$ 1613.1} & 80.2 {\scriptsize $\pm$ 53.3} \\
\midrule
\multirow{5}{*}{\begin{tabular}[c]{@{}l@{}}\textbf{Flow}\\ \textbf{Policy RL}\end{tabular}} 
& SAC Flow-T & 4 & 3457.1 {\scriptsize $\pm$ 54.1} & \third{5676.5 {\scriptsize $\pm$ 110.2}} & 5787.6 {\scriptsize $\pm$ 165.4} & \second{14897.8 {\scriptsize $\pm$ 132.7}} & \third{9821.1 {\scriptsize $\pm$ 241.5}} & \third{306937.8 {\scriptsize $\pm$ 1876.3}} & 101.5 {\scriptsize $\pm$ 11.2} \\
& SAC Flow-G & 4 & 3440.3 {\scriptsize $\pm$ 67.5} & 5094.1{\scriptsize $\pm$ 90.7} & 5847.1 {\scriptsize $\pm$ 151.3} & 12432.8 {\scriptsize $\pm$ 142.5} & 8163.3 {\scriptsize $\pm$ 201.7} & \second{368505.1 {\scriptsize $\pm$ 2248.2}} & 98.2 {\scriptsize $\pm$ 12.5} \\
& FlowRL & 1 & 2784.6 {\scriptsize $\pm$ 215.4} & 4564.7 {\scriptsize $\pm$ 312.1} & 5513.3 {\scriptsize $\pm$ 287.6} & 8963.4 {\scriptsize $\pm$ 198.5} & 5560.7 {\scriptsize $\pm$ 145.2} & 151870.6 {\scriptsize $\pm$ 1234.9} & 48.7 {\scriptsize $\pm$ 9.6} \\
& FPMD-R & 1 & \third{3491.4 {\scriptsize $\pm$ 61.2}} & 4134.5 {\scriptsize $\pm$ 596.5} & 5756.2 {\scriptsize $\pm$ 82.5} & 11043.6 {\scriptsize $\pm$ 612.2} & 6752.6 {\scriptsize $\pm$ 326.7} & 205336.8 {\scriptsize $\pm$ 1652.4} & 62.2 {\scriptsize $\pm$ 10.3} \\
& FPMD-M & 1 & 3063.4 {\scriptsize $\pm$ 693.2} & 4723.6 {\scriptsize $\pm$ 312.4} & 5762.4 {\scriptsize $\pm$ 152.5} & 9948.2 {\scriptsize $\pm$ 652.2} & 6227.9 {\scriptsize $\pm$ 621.1} & 168924.5 {\scriptsize $\pm$ 2045.1} & 55.4 {\scriptsize $\pm$ 9.6} \\
\midrule
\textbf{Proposed} & \textbf{SMFP} & 1 & \first{3825.8 {\scriptsize $\pm$ 29.1}} & \first{6776.5 {\scriptsize $\pm$ 117.4}} & \first{7518.3 {\scriptsize $\pm$ 77.1}} & \first{15769.3 {\scriptsize $\pm$ 125.5}} & \second{10245.7 {\scriptsize $\pm$ 93.5}} & \first{376132.6 {\scriptsize $\pm$ 1107.4}} & \first{137.2 {\scriptsize $\pm$ 11.3}}\\ 
\bottomrule 
\end{tabular}
\begin{tablenotes}
\normalsize
\item[*] DIME parameterises its critic with a distributional Q-function, which can provide richer value information {than the scalar Q-values commonly used in other methods}.
\end{tablenotes}
\end{threeparttable}}
\end{table*}

\textbf{{MuJoCo Environments and Training Details.}} 
We evaluate on the seven widely adopted Gym MuJoCo benchmarks~\citep{brockman2016openai, todorov2012mujoco} to align with established baselines.
{For SMFP, we report mean episodic rewards and standard deviations over 20 random seeds. Baseline entries combine reproduced and comparable published results, selected as detailed in Appendix~\ref{app:baseline_comparison}.}

\textbf{{MuJoCo Baselines.}} {We compare three policy classes: (1) Gaussian or deterministic policies with single-step inference (1-NFE): PPO~\citep{schulman2017PPO}, SPO~\citep{xie2025simple}, TD3~\citep{fujimoto2018addressing}, and SAC~\citep{haarnoja2018soft}. (2) Diffusion-based policies: DIPO~\citep{yang2023dipo}, QSM~\citep{psenka2024qsm}, QVPO~\citep{ding2024diffusion}, DIME~\citep{dime}, MaxEntDP~\citep{MaxEntDP}, and DPMD~\citep{maefficientDPMD}. (3) Flow-based policies: SAC Flow~\citep{zhang2025sac}, FlowRL~\citep{lv2025flow}, and FPMD~\citep{chen2025one}.}

\subsection{Comparative Performance Evaluation} 
{SMFP achieves the highest reported mean return on six of the seven MuJoCo tasks and ranks second on Humanoid, while using one-step action generation (Tab.~\ref{tab:main_results}). It also outperforms the Gaussian and deterministic baselines and the other one-step flow policies across the benchmark suite. These gains span locomotion and high-dimensional humanoid control, demonstrating effective policy optimisation across varied dynamics and action spaces with one-step generation. Appendix~\ref{app:mujoco_learning_curves} provides learning curves from reproduced runs (Fig.~\ref{fig:exp_result_curve}).}

\begin{figure}[t]
    \centering
    \includegraphics[width=\linewidth]{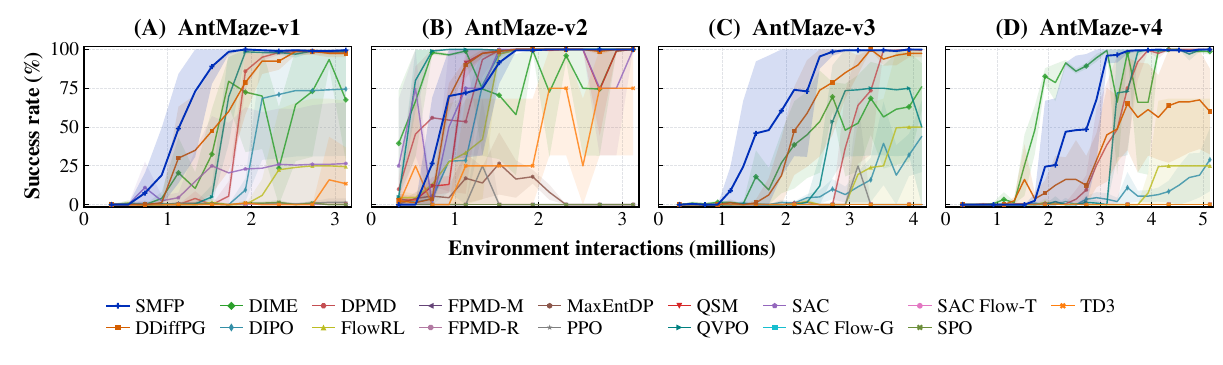}
    \caption{{Exploration in sparse-reward AntMaze tasks. Curves show mean success rates over 10 training seeds; shaded regions indicate one standard deviation. NovelD-style intrinsic rewards are used by all compared methods.}}
    \label{fig:antmaze_success_curves}
\end{figure}

\paragraph{Multimodal expressivity and exploration.}
Figure~\ref{fig:antmaze_success_curves} compares SMFP with Gaussian, diffusion, and flow baselines on four sparse-reward AntMaze tasks~\citep{li2024learning}, using shared NovelD-style intrinsic rewards~\citep{zhang2021noveld} and matched interaction budgets. SMFP achieves and sustains near-ceiling success across all four tasks, including AntMaze-v3, the most challenging task in our evaluation (see Appendix~\ref{app:antmaze} for task details). These results support effective exploration and reliable goal-reaching under sparse rewards, without the hierarchical trajectory clustering or mode-specific critics used by DDiffPG~\citep{li2024learning}. 
\subsection{Further Analysis}
\begin{figure}[t]
  \centering
  \begin{subfigure}[t]{0.34\textwidth}
    \centering
    \includegraphics[trim=0 0 0 0, clip, width=1\linewidth]{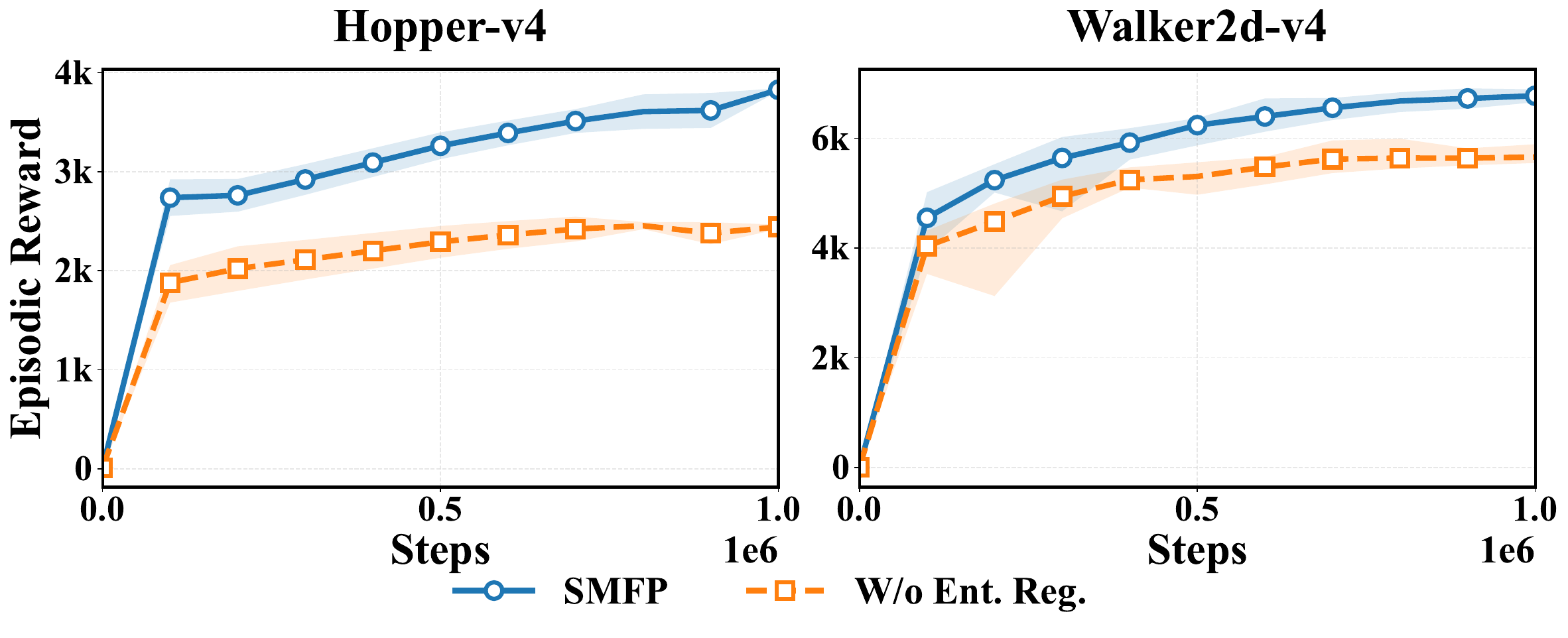}
    \caption{Comparison of SMFP and SMFP without the entropy regulariser. Curves show the mean over 5 seeds.}
    \label{fig:ablation_ent}
  \end{subfigure}
  \hfill
    \begin{subfigure}[t]{0.33\textwidth}
        \centering
        \includegraphics[trim=0 0 0 0, clip, width=\linewidth]{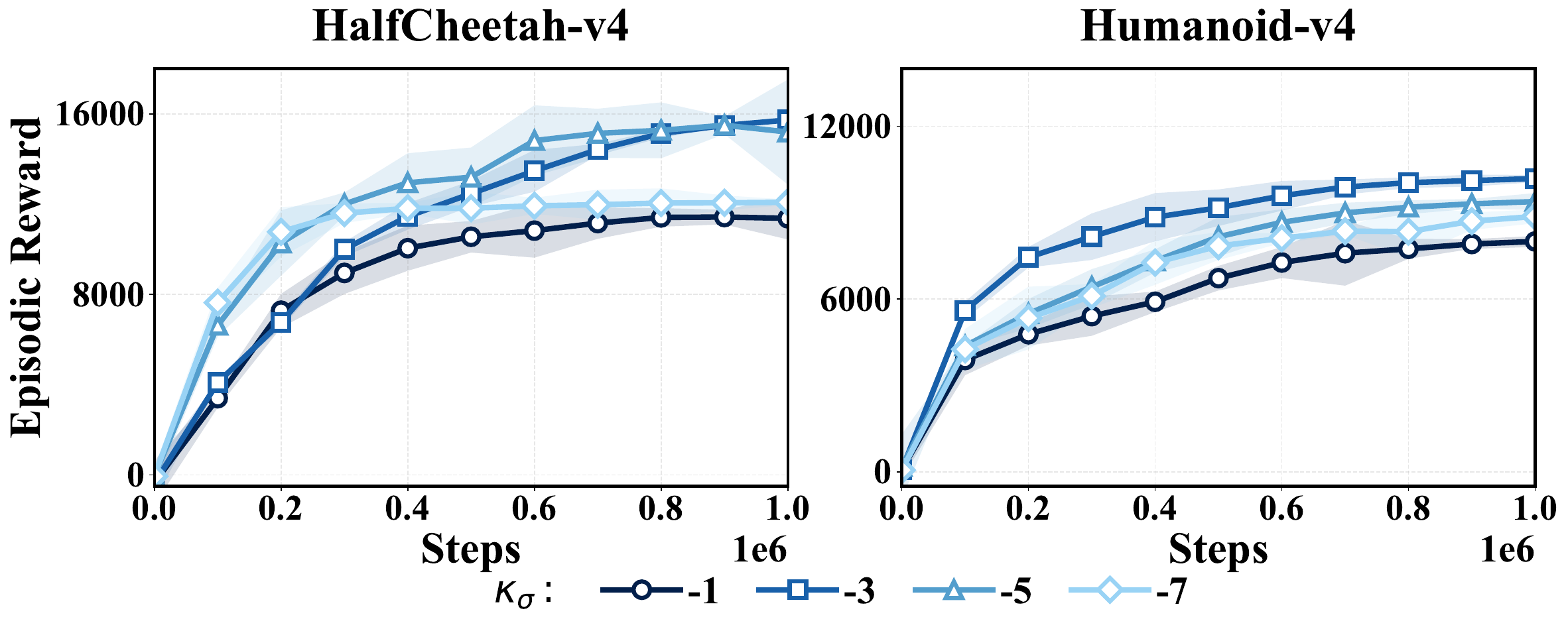} 
        \caption{Performance comparison under different $\kappa_{\sigma}$ settings. Curves show the mean over 5 seeds.}
        \label{fig:noise_scale}
    \end{subfigure}
    \hfill
  \begin{subfigure}[t]{0.29\textwidth}
    \centering
    \includegraphics[trim=0 0 0 0, clip, width=0.99\textwidth]{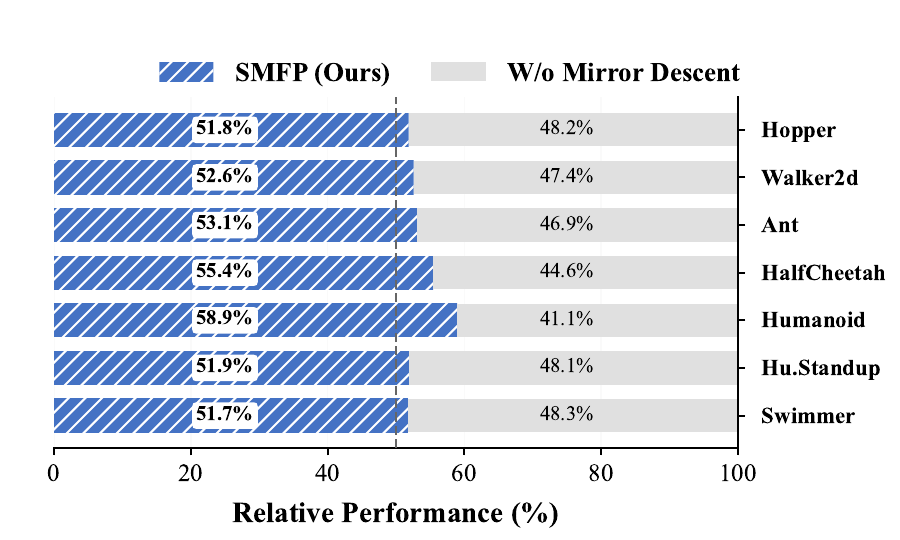} 
    \caption{{Relative performance of SMFP and its variant with mirror descent removed.}}
    \label{fig:ablation_md}
    \end{subfigure}
  \hfill
\end{figure}
\textbf{Effect of the entropy regulariser.}
{Removing the entropy regulariser slows learning and lowers final returns on both tasks in Fig.~\ref{fig:ablation_ent}. Explicit control of conditional stochasticity therefore improves learning beyond the diversity supplied by the latent mapping alone.}

\textbf{Sensitivity to the noise-scale threshold.}
Fig.~\ref{fig:noise_scale} evaluates the effect of the log-scale threshold
$\kappa_{\sigma}\in\{-1,-3,-5,-7\}$ in Eq.~\eqref{eq:entropy_term}. Performance
varies non-monotonically with this threshold: less negative values enforce
stronger stochasticity, while more negative values weaken the entropy constraint
and may reduce exploration. Overall, $\kappa_{\sigma}=-3$ achieves the best trade-off, indicating that the noise-scale threshold affects the balance between exploration and refinement.

\textbf{Effect of mirror-descent regularisation.}
We isolate the contribution of mirror descent by comparing SMFP with a variant
that removes the mirror-descent term. As shown in Fig.~\ref{fig:ablation_md},
SMFP consistently outperforms this variant across all benchmarks, with relatively
larger gains on high-dimensional tasks such as Humanoid. {These gains support value-weighted policy fitting for control.}

\begin{figure}[t]
    \vspace{-0.2cm}
    \centering
    \begin{subfigure}[t]{0.49\textwidth}
    \centering
    \includegraphics[trim=0 0 0 0, clip, width=1\textwidth]{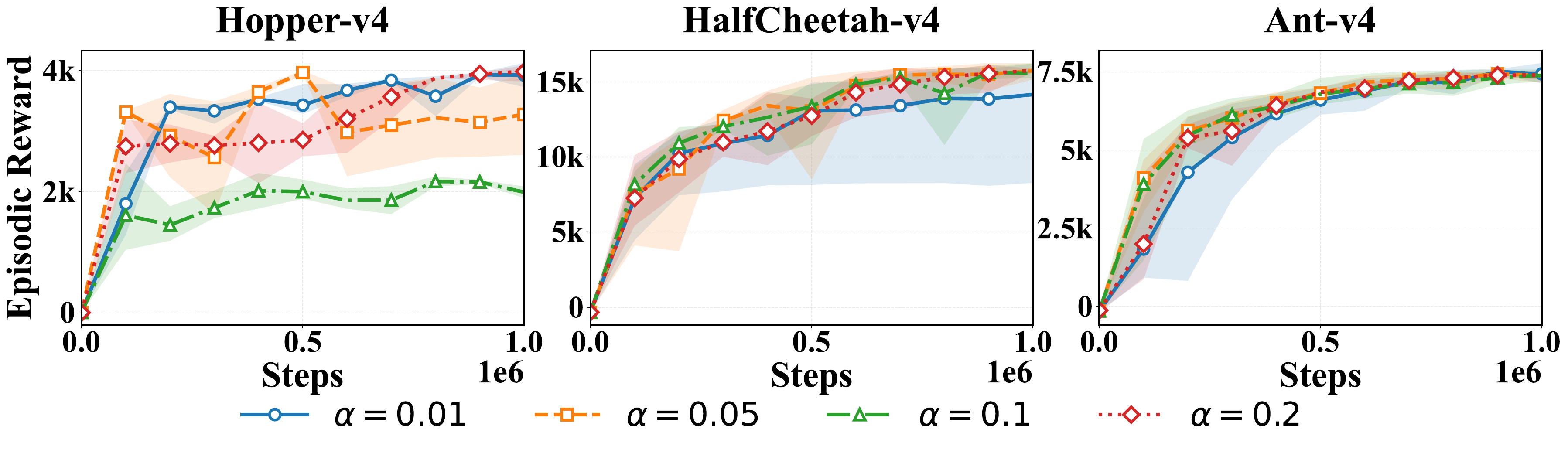} 
    \end{subfigure}
    \hfill
    \begin{subfigure}[t]{0.49\textwidth}
    \centering
    \includegraphics[trim=0 0 0 0, clip, width=1\textwidth]{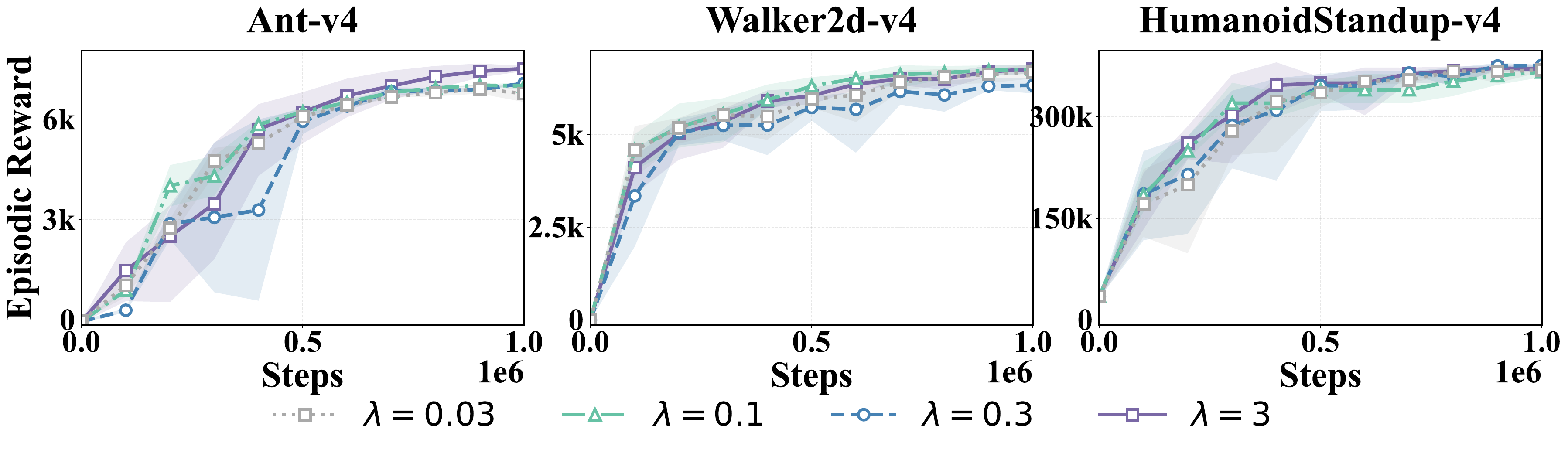} 
    \end{subfigure}
    \vspace{-0.2cm}
    \caption{Sensitivity analysis of the entropy temperature parameter $\alpha$ and the mirror descent regularisation coefficient $\lambda$ on MuJoCo benchmarks. Curves show the mean over 5 seeds.}
    \vspace{-0.5cm}
    \label{fig:hyper_ablation}
\end{figure}

\textbf{Sensitivity to optimisation hyperparameters.}
Fig.~\ref{fig:hyper_ablation} studies the sensitivity of SMFP to the entropy
temperature $\alpha$ and the mirror-descent coefficient $\lambda$. Overall, SMFP
is stable across a moderate range of settings. The entropy temperature $\alpha$
has a more visible effect on some tasks, such as Hopper, where overly large
values can slow policy refinement. {Performance is comparatively robust to $\lambda$ within the tested range. Appendix~\ref{sec:hyper} lists the full hyperparameter settings.}

\section{Conclusion}
\label{sec:conclusion}
In this work, we introduced Stochastic MeanFlow Policies (SMFP), a one-step generative policy framework for online off-policy reinforcement learning.
SMFP augments MeanFlow policies with an explicit stochastic reparameterisation and a tractable conditional entropy surrogate, enabling SAC-style exploration to be combined with off-policy mirror-descent optimisation. This design addresses the tension between expressive action modelling, efficient inference, and tractable entropy control.
Empirically, across seven MuJoCo benchmarks, SMFP achieves strong performance against both Gaussian and generative baselines with single-step inference. {Push-T and AntMaze further demonstrate multimodal action modelling, sparse-reward exploration, and reliable goal attainment.}


\clearpage
\subsection*{AI Use Statement}
{After completing the initial draft, we used LLM-based tools solely to correct grammar and improve the clarity and readability of the manuscript. Their use was limited to language editing and did not extend to research ideation, experimental design, scientific analysis, or the generation of substantive scientific content. All research contributions, analyses, and conclusions are the authors’ own. No LLM is credited with authorship, and the authors take full responsibility for the accuracy and integrity of the entire manuscript, including all passages revised with LLM assistance.}

\subsection*{{Reproducibility Statement}}
{To facilitate reproducibility of our work, we provide detailed descriptions of the experimental environments, hyperparameters, and other experimental details used in our study in Section~\ref{sec:experiments} and the appendices. Our code and model checkpoint will be released publicly upon acceptance to further support reproducibility and future research.}

\bibliographystyle{iclr2027_conference}
\bibliography{reference}

\clearpage


\appendix

\section{Limitations and Future Work}
Although the hybrid optimisation framework is broadly applicable in principle, the tractable surrogate objective developed here relies on the reparameterised structure of the MeanFlow policy. Extending it to other generative policy classes may therefore require policy-specific entropy bounds and regularisation schemes.
Within this framework, we keep $\alpha$ and $\lambda$ fixed throughout training to avoid confounding scheduling effects and to isolate the contribution of the proposed objective and policy parameterisation.
{Adaptive weighting of $\alpha$ and $\lambda$ is a natural direction for extending this framework.} We will also investigate stronger critic architectures, such as CrossQ~\citep{crossq} and distributional Q-learning~\citep{bellemare2017distributional}, for improving value estimation and overall training behaviour.

\section{SMFP Training Procedure}
\label{app:training_procedure}
\subsection{Implementation Details}
\label{app:algorithm}
\textbf{Value-Guided Action Selection.}
{Following the value-guided candidate selection used in QVPO and MeanFlowQL~\citep{ding2024diffusion,meanflowql}, we apply Best-of-$K$ selection to favour high-value actions. For each state, SMFP generates $K$ candidates in parallel using JAX \texttt{vmap}, with one network evaluation per candidate. The critic then selects the action:}
\begin{equation}
a_\theta^K(s) \triangleq 
\underset{a \in \{a_1, \dots, a_K \sim g_\theta(s,e,b=0,t=1)\}}{\operatorname{argmax}} Q(s,a).
\label{eq:23}
\end{equation}
{This selection operates on completed action samples; each candidate retains single-step generation without iterative denoising or ODE integration. The candidate-size ablation in Fig.~\ref{fig:ablation_k} characterises its effect, and Tab.~\ref{tab:time_antv3_highlight} reports the computational cost of SMFP.}
{Following QVPO~\citep{ding2024diffusion}, we use separate candidate counts for behaviour inference and target estimation, denoted by $K_b$ and $K_t$, and choose a smaller target pool with $K_t=0.5K_b$.}

\textbf{Critic Implementation.}
We train the critic using the standard target-Q architecture. Given critic parameters $\phi$ and target parameters $\bar{\phi}$, the critic is updated by minimising the Bellman error:
\begin{equation}
\label{eq:critic_loss}
    \mathcal{L}_{\text{Critic}}(\phi)
    =
    \mathbb{E}_{(s,a,r,s') \sim \mathcal{D}}
    \left[
    \left(
    Q_\phi(s,a) - \bar{Q}
    \right)^2
    \right],
\end{equation}
where the target is computed using the target critic with stop-gradient:
\begin{equation}
\label{eq:critic_target}
    \bar{Q}
    =
    r
    +
    \gamma
    \left[
    Q_{\bar{\phi}}(s',a')
    +
    \alpha
    \sum_{i=1}^{d}
    \log \sigma_\theta^{(i)}(s',e)
    \right]_{\mathrm{sg}} .
\end{equation}
Here, $a'$ is generated by the current one-step MeanFlow policy at $s'$, and the entropy term follows the lower bound in Proposition~\ref{prop:entropy_lb}. The target parameters $\bar{\phi}$ are updated by an exponential moving average. {Algorithm~\ref{alg:SMFP} in the main text summarises the training procedure.}

\section{Proofs}
\subsection{Theoretical Derivations for Entropy Bounds}
\label{app:entropy_proofs}
{The stochastic reparameterisation in Eq.~\ref{eq:reparam_action} separates the latent input $e$ from the independent Gaussian perturbation $\epsilon$. At a fixed state and input $(a_t,b,t)$, randomness in $\epsilon$ induces the conditional density $\pi_\theta(\cdot\mid a_t,b,t)=\mathcal{N}(\mu_\theta,\Sigma_\theta)$, with mean $\mu_\theta=a_t-u_\theta(a_t,b,t)$ and covariance $\Sigma_\theta=\operatorname{diag}(\sigma_\theta^2)$. For one-step action generation, $a_t=e$, $b=0$, and $t=1$, so conditioning on $e$ fixes the mean and scale of this Gaussian. We suppress the state argument below and use differential entropy throughout.}

{Integrating over $e$ produces a continuous mixture of these conditional Gaussians. Its means and scales vary with the latent input, allowing the marginal policy to represent multiple action modes. The resulting marginal density is implicit, whereas each conditional entropy is available analytically. The following decomposition identifies the component of marginal entropy that SMFP can regulate directly:}

\begin{proposition}[\textbf{Tractable Entropy Lower Bound}]
\label{prop:entropy_lb}
{At a fixed state, consider the one-step policy with $a_t=e$, $b=0$, and $t=1$, positive conditional standard deviations, and finite entropies. The latent variable $e\sim\mathcal{N}(0,I)$ is distinct from Euler's number $\mathrm{e}$.} 

{The marginal entropy decomposes into the expected entropy of the conditional action distribution and the mutual information between the generated action and $e$. In this notation, $\mathcal{H}(\pi_\theta\mid e)$ includes the expectation over latent inputs. Since mutual information is non-negative,}
\begin{equation}
    \mathcal{H}(\pi_\theta) = \mathcal{H}(\pi_\theta \mid e) + \underbrace{I_\theta(\pi_\theta; e)}_{\ge 0} \ge \mathcal{H}(\pi_\theta \mid e).
\end{equation}
The conditional term affords an analytical solution based on the predicted Gaussian parameters:
\begin{equation}
\begin{aligned}
        &\mathcal{H}(\pi_\theta \mid e) = \mathbb{E}_{e}\left[\sum_{i=1}^{d} \log \sigma_\theta^{(i)}(a_t, b, t)\right] + \underbrace{\frac{d}{2}\log(2\pi {\color{hiccup}\mathrm{e}})}_{const.}\\
        \Rightarrow \quad &\tilde{\mathcal{H}}(\pi_\theta \mid e) = \mathbb{E}_{e}\left[\sum_{i=1}^{d} \log \sigma_\theta^{(i)}(a_t, b, t)\right].
        \label{eq:tractable_entropy}
\end{aligned}
\end{equation}
{The additive Gaussian constant is independent of $\theta$, so the expected sum of log-standard deviations supplies the parameter-dependent part of this lower bound. The mutual-information term captures the dependence of generated actions on the latent input. SMFP therefore combines latent-dependent action generation with direct control of the conditional dispersion around each generated mean.}
\end{proposition}

{Equation~\eqref{eq:entropy_term} implements this control using the average log-standard deviation across action dimensions. This quantity is the logarithm of the geometric mean of the conditional standard deviations, so the threshold acts on their aggregate scale while permitting different scales in different dimensions. Below $\kappa_\sigma$, minimising the hinge loss increases this average log-scale. Once the threshold is reached, that sample contributes no further expansion pressure from the entropy term.}

{The hinge is applied before the expectation in Eq.~\eqref{eq:entropy_term}, allowing the penalty to respond to individual sampled inputs with low conditional dispersion. It is a soft regulariser whose strength is set by $\alpha$ in the actor objective. This construction makes the conditional scale directly trainable without evaluating the mutual-information term or the marginal density. The performance gains from the entropy regulariser and the sensitivity to $\kappa_\sigma$ in Figs.~\ref{fig:ablation_ent}--\ref{fig:noise_scale} support this form of stochasticity control.}

{The same surrogate has a diffuse-prior interpretation. For the conditional Gaussian policy and $p_0=\mathcal{N}(\mu_0,\tau^2 I)$, the parameter-dependent part of the KL divergence separates conditional entropy from a second-moment penalty.}

\begin{proposition}[\textbf{Max-Entropy Derivation}]
\label{prop:kl_limit}
{For each fixed conditional Gaussian policy with finite second moments, the parameter-dependent KL objective converges to its negative entropy, up to an additive constant, as $\tau\to\infty$. Averaging over $e$ yields the conditional-entropy surrogate in Proposition~\ref{prop:entropy_lb}.}
\end{proposition}

\begin{proof} 
{For the conditional policy, the Gaussian KL consists of a negative log-scale term, a squared distance between means, and a covariance term. Substituting the mean $a_t-u_\theta$ and removing the additive term $d\log\tau-d/2$, which is independent of $\theta$, gives:}
\begin{equation}
    {\mathcal{J}(\theta) = -\sum_{i} \log \sigma_\theta^{(i)} + \frac{1}{2\tau^2} \left( \|a_t-u_\theta - \mu_0\|^2 + \|\sigma_\theta\|_F^2 \right).}
\end{equation}
{The two squared norms measure the conditional second moment about the prior mean. Increasing $\tau$ reduces their contribution relative to the negative log-scale term. For fixed $\theta$, this finite second-moment contribution vanishes as $\tau\to\infty$, giving:}
\begin{equation}
    {\lim_{\tau \to \infty} \mathcal{J}(\theta) = -\sum_{i} \log \sigma_\theta^{(i)}.}
\end{equation}
{If the second moment is integrable over $e$ and the expected log-scales are finite, the same decomposition holds after averaging over latent inputs. The limiting parameter-dependent objective is then the negative conditional-entropy surrogate in Proposition~\ref{prop:entropy_lb}. This limit describes the objective for fixed policy parameters after the prior-dependent additive term has been removed.}

{The diffuse-prior calculation explains why the log-scale term favours conditional dispersion. The practical hinge loss uses this direction of regularisation below a prescribed threshold, while allowing the value objective to guide action quality. Thus, the lower-bound argument identifies the computable entropy component, and the diffuse-prior limit explains the associated regularisation.} 
\end{proof}

\subsection{Derivation of the Stochastic MeanFlow Objective}
\label{app:derivation_smf}
In this section, we derive the training objective for Stochastic MeanFlow. Our goal is to extend the deterministic MeanFlow identity to support a stochastic generative policy $g_\theta$, parameterised by a residual flow with learnable variance.

Following the reformulation principles {introduced in MeanFlowQL~\citep{meanflowql}}, we define the stochastic policy $g_\theta(a_t, b, t)$ by decomposing the trajectory into a deterministic residual and a stochastic perturbation, and obtain the parameterisation:
\begin{equation}
    g_\theta(a_t, b, t) = a_t - u_\theta(a_t, b, t) + \sigma_\theta(a_t, b, t) \odot \epsilon, \quad \text{where } \epsilon \sim \mathcal{N}(0, I).
    \label{eq:stochastic_g}
\end{equation}
Rearranging Eq.~\ref{eq:stochastic_g}, we express the underlying velocity generator $u_\theta$ as:
\begin{equation}
    u_\theta(a_t, b, t) = a_t - g_\theta(a_t, b, t) + \sigma_\theta(a_t, b, t) \odot \epsilon.
    \label{eq:u_def}
\end{equation}
A fundamental property of the MeanFlow {framework~\citep{meanflow}} is the consistency between the generator $u_\theta$ and the vector field $v(a_t, t)$. This is governed by the \emph{MeanFlow Identity}:
\begin{equation}
    u_\theta(a_t, b, t) = v(a_t, t) - (t - b) \frac{d}{dt} u_\theta(a_t, b, t).
    \label{eq:meanflow_core}
\end{equation}
Substituting the expression for $u_\theta$ (Eq.~\ref{eq:u_def}) into Eq.~\ref{eq:meanflow_core}, we obtain:
\begin{equation}
    a_t - g_\theta + \sigma_\theta \odot \epsilon = v(a_t, t) - (t - b) \frac{d}{dt} \left( a_t - g_\theta + \sigma_\theta \odot \epsilon \right).
\end{equation}
{Solving for $g_\theta$ and using the flow dynamics $\frac{d a_t}{dt} = v(a_t, t)$, we obtain:}
\begin{equation}
\begin{aligned}
    g_\theta &= a_t + \sigma_\theta \odot \epsilon - v(a_t, t) + (t - b) \left[ \frac{d a_t}{dt} - \frac{d g_\theta}{dt} + \frac{d \sigma_\theta}{dt} \odot \epsilon \right] \\
    &= a_t + \sigma_\theta \odot \epsilon - v(a_t, t) + (t - b)v(a_t, t) - (t - b)\left( \frac{d g_\theta}{dt} - \frac{d \sigma_\theta}{dt} \odot \epsilon \right) \\
    &= a_t + (t - b - 1)v(a_t, t) + \sigma_\theta \odot \epsilon - (t - b)\left( \frac{d g_\theta}{dt} - \frac{d \sigma_\theta}{dt} \odot \epsilon \right).
\end{aligned}
\label{eq:g_expanded}
\end{equation}

{For these time derivatives, we hold the sampled perturbation $\epsilon$ fixed and differentiate along the trajectory $a_t$, with $b$ fixed. The derivative of $\sigma_\theta\odot\epsilon$ is therefore $(d\sigma_\theta/dt)\odot\epsilon$. Applying the chain rule to a differentiable function $f(a_t,b,t)$ gives:}
\begin{equation}
    \frac{d}{dt} f(a_t, b, t) = \nabla_{a_t} f \cdot \frac{d a_t}{dt} + \nabla_b f \cdot \frac{d b}{dt} + \nabla_t f \cdot \frac{d t}{dt}.
\end{equation}
Given the flow dynamics $\frac{d a_t}{dt} = v(a_t, t) = e-a$, $\frac{d b}{dt} = 0$, and $\frac{d t}{dt} = 1$, this corresponds to the {Jacobian-vector product (JVP)} operator:
\begin{equation}
    \frac{d}{dt} f(a_t, b, t) = \texttt{jvp}\big(f, (a_t, b, t), (v(a_t, t), 0, 1)\big).
\end{equation}
Applying this operator to both $g_\theta$ and $\sigma_\theta$ in Eq.~\ref{eq:g_expanded}, we derive the explicit regression target $g_{\text{tgt}}$:
\begin{equation}
    g_{\text{tgt}} \triangleq a_t + (t - b - 1) v(a_t, t) + \sigma_\theta \odot \epsilon - (t - b) \left[ \texttt{jvp}(g_\theta) - \texttt{jvp}(\sigma_\theta) \odot \epsilon \right].
\end{equation}

\paragraph{Training Objective.}
Finally, to enforce the derived identity, we minimise the discrepancy between the predicted generator $g_\theta$ and the target $g_{\text{tgt}}$. We employ the Huber loss $\mathcal{L}_\delta$ for robustness against outliers and apply the stop-gradient operator $\operatorname{sg}(\cdot)$ to the target to prevent unstable feedback loops during training. The Stochastic MeanFlow objective is formulated as:
\begin{equation}
    \mathcal{L}_{\text{SMF}}(\theta) = \mathbb{E}_{t, b, \epsilon} \left[ \mathcal{L}_{\delta}\left( g_\theta(a_t, b, t) - \operatorname{sg}(g_{\text{tgt}}) \right) \right],
\end{equation}
where $\mathcal{L}_\delta$ denotes the Huber loss with threshold $\delta=1.0$:
\begin{equation}
    \mathcal{L}_\delta(x) =
    \begin{cases}
    \frac{1}{2} x^2 & \text{if } |x| \le \delta, \\
    \delta (|x| - \frac{1}{2}\delta) & \text{otherwise}.
    \end{cases}
\end{equation}
{The scale-derivative term in the target accounts for the variation of the Gaussian perturbation along the flow trajectory. This preserves the residual MeanFlow relation when the policy is written in the stochastic form of Eq.~\eqref{eq:stochastic_g}. The stop-gradient operation makes the evaluated target fixed during each regression update.}

{In the full actor objective, the weighted SMF loss fits the residual mapping to actions favoured by the critic. The direct Q-gradient also passes through the reparameterised action, so it can shape both the mean and the conditional scale. The entropy-floor penalty supplies an additional gradient on the scale when conditional dispersion is low. These roles explain how value-directed generation and explicit stochasticity control are combined in Eq.~\eqref{eq:overall_loss_tractable}.}

\subsection{Justification for Truncated Advantage Weighting.}
\label{app:addvantage}
{The value-based mirror-descent target in Eq.~\eqref{eq:pmd_closed_form} increases the probability of high-value actions relative to the old policy. To explain the resulting weights, we first consider a fixed state and hold the critic and old policy fixed in the projection argument. Assuming a finite partition function, the target is:}
\begin{equation}
    \pi^*(a|s) = \frac{1}{Z(s)} \pi_{\text{old}}(a|s) \exp\left(\frac{Q^{\pi_{\text{old}}}(s, a)}{\lambda}\right),
    \label{eq:target_dist_detailed}
\end{equation}
where $\lambda$ is the regularisation coefficient and $Z(s)$ is the partition function normalising the distribution.

{We use a forward KL projection, $\mathcal{D}_{\mathrm{KL}}(\pi^*\|\pi_\theta)$, to motivate weighted policy fitting. The target entropy is independent of the parameters being fitted, so minimising this divergence reduces to minimising the expected negative log-likelihood under $\pi^*$. Using the old policy as a proposal expresses this expectation through a density ratio:}
\begin{equation}
\begin{aligned}
    \mathcal{L}_{\text{MD}}(\theta) &= \mathbb{E}_{a \sim \pi^*} \left[ - \log \pi_\theta(a|s) \right] \\
    &= \mathbb{E}_{a \sim \pi_{\text{old}}} \left[ \frac{\pi^*(a|s)}{\pi_{\text{old}}(a|s)} \left( - \log \pi_\theta(a|s) \right) \right] \\
    &= \mathbb{E}_{a \sim \pi_{\text{old}}} \left[ \frac{\frac{1}{Z(s)}\pi_{\text{old}}(a|s) \exp\left(\lambda^{-1} Q(s, a)\right)}{\pi_{\text{old}}(a|s)} \left( - \log \pi_\theta(a|s) \right) \right] \\
    &= \mathbb{E}_{a \sim \pi_{\text{old}}} \left[ \frac{1}{Z(s)} \exp\left( \frac{Q(s, a)}{\lambda} \right) \left( - \log \pi_\theta(a|s) \right) \right].
\end{aligned}
\label{eq:derivation_step1}
\end{equation}
{The ratio cancels the old-policy density, leaving an exponential weight determined by the critic. Centring by $V(s)$ expresses the same within-state preference in terms of advantage: an action-independent baseline multiplies all exponential weights at that state by a common factor. With $A^{\pi_{\text{old}}}(s,a)=Q(s,a)-V(s)$,}
\begin{equation}
\begin{aligned}
    \frac{1}{Z(s)} \exp\left( \frac{Q(s, a)}{\lambda} \right) 
    &= \frac{1}{Z(s)} \exp\left( \frac{A^{\pi_{\text{old}}}(s, a) + V(s)}{\lambda} \right) \\
    &\propto \exp\left( \frac{A^{\pi_{\text{old}}}(s, a)}{\lambda} \right).
\end{aligned}
\end{equation}
{The proportionality describes relative action weights at a fixed state; the normalisation factor can vary across states. For the practical generative-policy update, we use these advantage-based weights with the tractable $\mathcal{L}_{\mathrm{SMF}}(\theta;s,a)$ in place of marginal likelihood fitting. The resulting MeanFlow regression surrogate retains the target preference for high-value actions:}
\begin{equation}
    \mathcal{L}(\theta) = \mathbb{E}_{s \sim \mathcal{D}, a \sim \pi_{\text{old}}} \left[ \underbrace{\exp\left( \frac{A^{\pi_{\text{old}}}(s, a)}{\lambda} \right)}_{w(s, a)} \mathcal{L}_{\text{SMF}}(\theta; s, a) \right].
    \label{eq:theoretical_exp_loss}
\end{equation}
{Exponential weights can concentrate the regression update on a small number of samples when advantages are large or $\lambda$ is small. Since the advantages are estimated by an online critic, this concentration can also magnify value-estimation errors. A weighting rule with linear growth offers a more moderate response to large positive estimates while retaining a preference for actions above the value baseline.}

{The first-order expansion around $A(s,a)=0$ gives a local motivation for linear dependence on advantage:}
\begin{equation}
    \exp\left( \frac{A(s, a)}{\lambda} \right) \approx 1 + \frac{A(s, a)}{\lambda}.
\end{equation}
{The expansion contains a uniform fitting component and an advantage-dependent component. Following QVPO~\citep{ding2024diffusion}, the practical update retains the positive-advantage component. This choice changes the regression weights and concentrates fitting on actions above the value baseline.}

{Using signed advantages directly would give negative coefficients to a non-negative regression loss. Rectification instead assigns zero weight to actions at or below the baseline and a positive weight proportional to the estimated improvement for the remaining actions:}
\begin{equation}
w(s, a) = \max\left( 0, Q(s, a) - V(s) \right) \propto \left( A^{\pi_{\text{old}}}(s, a) \right)^+.
\label{eq:final_relu_weight}
\end{equation}
{The baseline is estimated by averaging the critic values of $N_{\mathrm{adv}}$ actions drawn directly from the pre-update policy, as described in Sec.~\ref{sec:md_method}. This provides a state-dependent reference using the existing critic, without training a separate value network. Substituting the positive-advantage weights gives the practical regression objective:}
\begin{equation}
    \mathcal{L}_{\text{Final}}(\theta) = \mathbb{E}_{s \sim \mathcal{D}, a \sim \pi_{\text{old}}} \left[ \left( A^{\pi_{\text{old}}}(s, a) \right)^+ \mathcal{L}_{\text{SMF}}(\theta; s, a) \right].
\end{equation}
{The positive-advantage weighting has the following useful properties:}

\noindent \textbf{1. Monotonicity.} 
{Weights are zero at or below the value baseline and strictly increasing for positive advantages. Higher-valued actions therefore receive larger regression weights.}

\noindent \textbf{{2. Linear Weight Growth.}} 
{For positive advantages, increasing $Q-V$ changes the weight linearly. This avoids the exponential amplification of large positive critic estimates. The Huber loss addresses a different source of sensitivity: its derivative with respect to the regression residual stops growing once the residual enters the linear regime. Positive-advantage weighting and Huber regression therefore moderate sensitivity to value estimates and fitting errors, respectively.}

\noindent \textbf{{3. State-Dependent Value Centring.}} 
{Subtracting $V(s)$ makes action selection depend on improvement relative to the current state. Adding the same state-dependent offset to $Q$ and $V$ leaves the weights unchanged. A positive rescaling of both quantities preserves which actions receive non-zero weight and their relative weights, while rescaling the regression term. In the full actor objective, $\lambda$ controls the strength of this term relative to the direct Q loss and entropy penalty.}

{The weighting supplies value guidance through regression targets, complementing the direct Q-gradient through generated actions in Eq.~\eqref{eq:overall_loss_tractable}.}

\subsection{Discussion: Entropic Mirror Descent with Generative Policies}
\label{subsec:theoretical_justification}

{SAC-style entropy regularisation and mirror descent provide complementary guidance for policy improvement. The value-based MD target reweights the old policy towards actions with higher critic values, while entropy regularisation encourages stochastic action generation. When several separated action regions have high value and support under the old policy, the MD target can retain multiple modes. A latent generative mapping can represent these regions, and the conditional Gaussian scale gives SMFP a tractable way to regulate dispersion around its generated means.}

{The practical actor objective combines these roles through three gradient paths. Advantage-weighted MeanFlow regression fits actions favoured by the critic, the direct Q term provides differentiable value guidance, and the entropy-floor penalty acts when the conditional scale falls below its threshold. The scale remains learnable under the value objective after this penalty becomes inactive. This allows SMFP to adapt action quality and stochasticity within the same one-step policy.}

{The action-distribution visualisation in Fig.~\ref{fig:intro_network}b illustrates this distinction. In the displayed setting, the Gaussian policy concentrates on a dominant mode, and the conventional flow policy captures multiple modes with increasingly concentrated samples. SMFP retains broader coverage around the target modes. The entropy ablation in Fig.~\ref{fig:ablation_ent} provides complementary task-level evidence: removing the regulariser slows learning and lowers final returns. Together, the visualisation and ablation support combining generative expressivity with explicit control of conditional stochasticity.}

\section{Additional Experimental Results}
\label{app:additional_results}
\subsection{{Exploration in Sparse-Reward AntMaze Tasks}}
\label{app:antmaze}
\paragraph{{Tasks and environment.}}
{We evaluate AntMaze-v1--v4 from \citet{li2024learning} in a Gymnasium/MuJoCo port of the released task definitions. The port retains the maze layouts, designated start and goal locations, and episode horizons: 500 steps for v1--v2 and 700 for v3--v4. Each task has sparse goal rewards; v2--v4 have two possible goals. Figure~\ref{fig:antmaze_layouts} shows the four layouts.}

\begin{figure}[H]
    \centering
    \includegraphics[width=\linewidth]{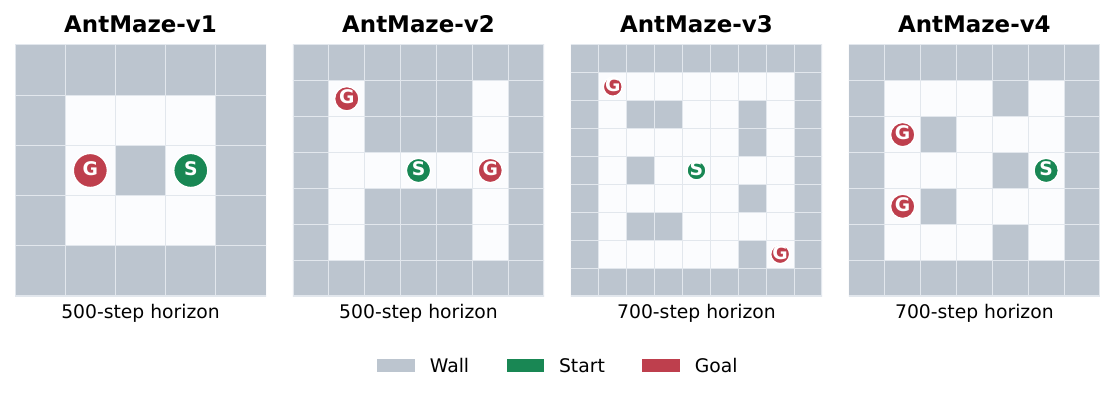}
    \caption{{AntMaze task layouts, redrawn from the released definitions~\citep{li2024learning}. Grey cells are walls; S and G mark designated start and goal cells. The panels show maze geometry rather than sampled trajectories. Episode horizons are 500 steps for v1--v2 and 700 for v3--v4.}}
    \label{fig:antmaze_layouts}
\end{figure}

\paragraph{{Reproduction protocol.}}
{We compare SMFP and the baselines in this port with matched environment-interaction budgets. Figure~\ref{fig:antmaze_success_curves} in the main text reports mean success rates over 10 training seeds for every method, with shaded regions showing one standard deviation. All compared AntMaze methods use a NovelD-style intrinsic reward during training~\citep{zhang2021noveld}. SMFP retains its one-step stochastic policy, conditional-entropy regularisation, and off-policy mirror-descent updates, without trajectory clustering, mode-specific critics, multimodal training batches, or mode embeddings.}

\paragraph{{Goal-reaching results.}}
{We assess exploration through the acquisition and retention of goal-reaching behaviour. Figure~\ref{fig:antmaze_success_curves} shows that SMFP learns successful navigation across all four mazes and sustains high success during later training. The learning trajectories support effective exploration and reliable goal attainment under the shared intrinsic-reward protocol. The Push-T visualisation in Fig.~\ref{fig:intro_network}b separately examines the representation of distinct action modes.}

\subsection{{MuJoCo Learning Curves}}
\label{app:mujoco_learning_curves}
{Figure~\ref{fig:exp_result_curve} complements the aggregate results in Table~\ref{tab:main_results} by showing learning dynamics across the seven MuJoCo benchmarks. The curves use reproduced runs under a common training budget and evaluation protocol.}

\begin{figure}[H]
    \centering
    \includegraphics[trim=0 0 85 0, clip, width=\textwidth]{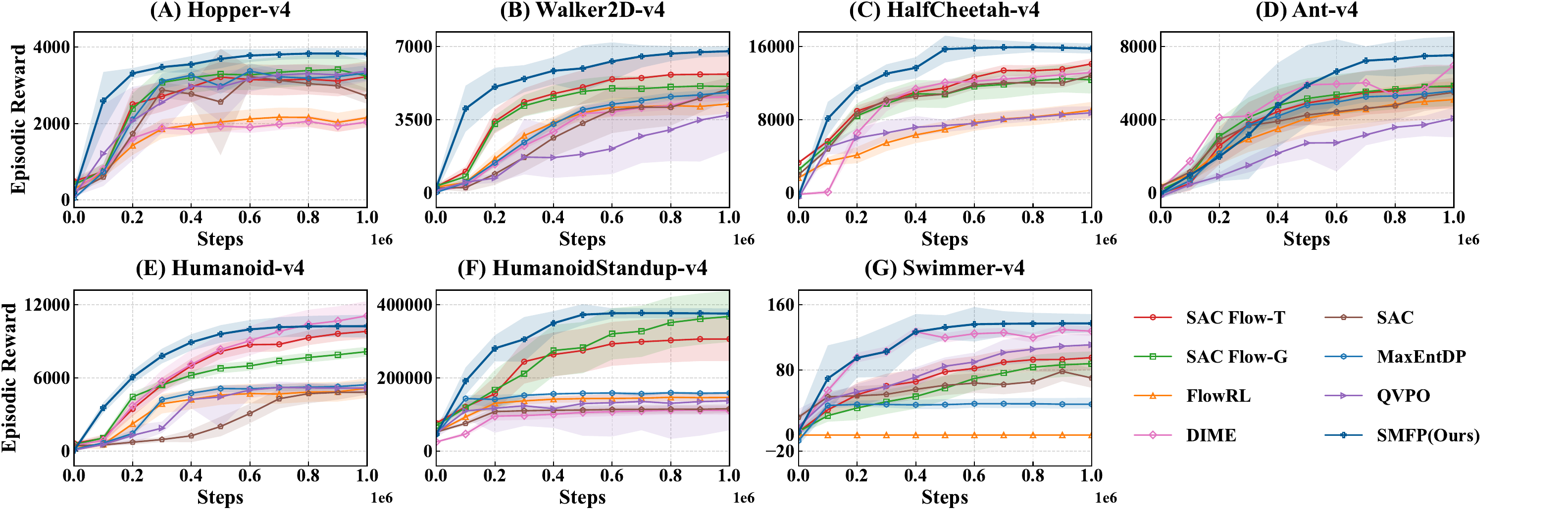}
    \caption{{Comparison of evaluation performance across 7 benchmarks. All methods are trained for 1M environment steps. The x-axis denotes environment steps, and the y-axis denotes episodic reward.}}
    \label{fig:exp_result_curve}
\end{figure}

\subsection{Hyperparameter Sensitivity and Analysis}
\begin{figure}[H]
  \centering
    \begin{subfigure}[t]{0.46\textwidth}
    \centering
    \includegraphics[trim=0 0 0 0, clip, width=1.0\textwidth]{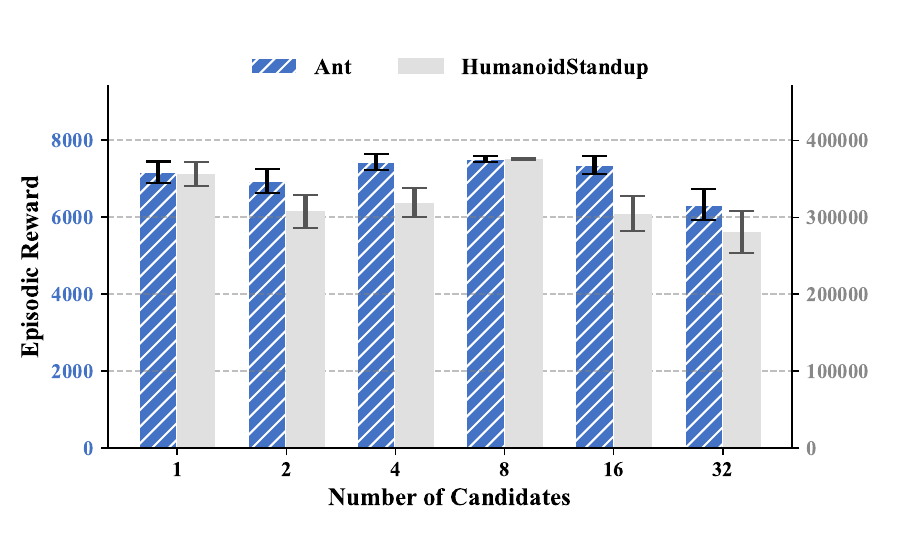} 
    \caption{{Impact of the number of candidates $K$ on performance across two tasks using best-of-$K$ action selection.}}
    \label{fig:ablation_k}
    \end{subfigure}
    \hfill
    \begin{subfigure}[t]{0.52\textwidth}
    \centering
    \includegraphics[trim=0 0 0 0, clip, width=\linewidth]{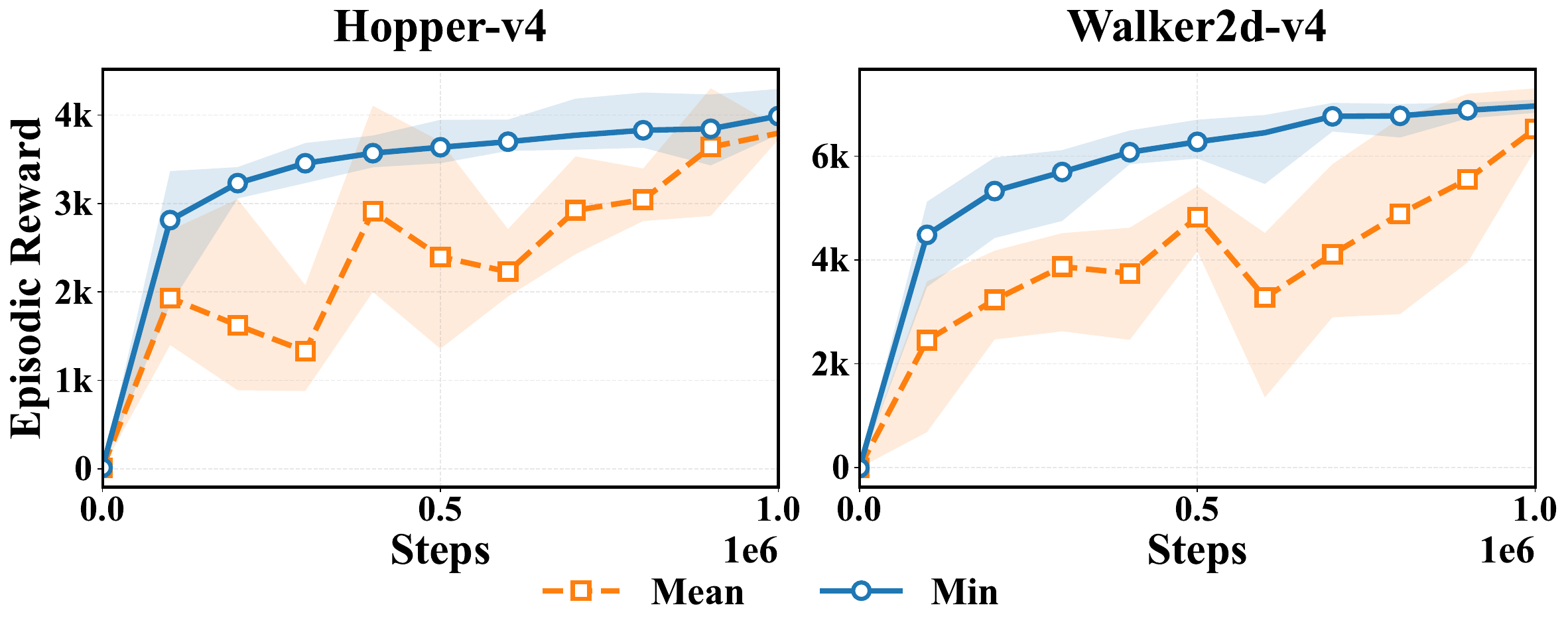} 
    \caption{{Sensitivity analysis of Q-network aggregation strategies.}}
    \label{fig:q_agg} 
    \end{subfigure}

  \vspace{-0.15cm}
\end{figure}

\textbf{Effect of candidate size $K$ in value-guided selection.}
{Fig.~\ref{fig:ablation_k} illustrates the impact of candidate pool size $K$. Among the tested settings, $K=8$ gives the strongest performance on these two tasks. Larger pools (e.g., $K=32$) can reduce returns, plausibly because maximising over more candidates amplifies critic estimation errors. These results motivate our default choice of $K=8$.}

\textbf{Sensitivity to Q-value aggregation.} We further analyse the target Q-value aggregation strategy \texttt{q\_agg}. As
shown in Fig.~\ref{fig:q_agg}, using the \textit{min} operator consistently
outperforms the \textit{mean} operator. This suggests that mean aggregation is
less robust in our maximum-entropy framework, whereas the conservative
\textit{min} aggregation provides more stable target estimates and improves final
performance. Accordingly, we use \textit{min} as the default Q-aggregation
strategy in SMFP.

\subsection{Training and Inference Time. }
{Tab.~\ref{tab:time_antv3_highlight} benchmarks computational efficiency on Ant-v4. SMFP attains the highest return among the listed methods with an inference latency of 0.42\,ms. One-step generation and parallel candidate evaluation through JAX's \texttt{vmap} keep best-of-$K$ selection efficient, yielding more than an order-of-magnitude lower inference latency than the listed iterative diffusion baselines. Training-time gains are smaller: SMFP also samples from the old policy and optimises the MeanFlow objective. The results therefore support a strong return--latency trade-off, while training remains more expensive than for the simple Gaussian and deterministic baselines.}
\begin{table}[ht]
    \centering
        \caption{{Training and inference efficiency on Ant-v4. SMFP achieves the highest return among the listed methods with low inference latency.}}
    \label{tab:time_antv3_highlight}
    \resizebox{\linewidth}{!}{ 
    \begin{tabular}{lccccccccccc} 
        \toprule
        Method & {\color{hiccup}SMFP}  & QVPO & DIME & DIPO & TD3 & SAC & PPO & SPO & QSM & MaxEntDP \\
        \midrule
        Training Time (h)   & 2.7  & 5.2 & 3.5 & 9.6 & 0.4 & 2.2 & 0.3 & 0.3 & 0.8 & 2.3 \\
        Inference Time (ms) & {0.42}  & 5.77 & 4.68 & 5.20 & 0.27 & 0.13 & 0.22 & 0.32 &  5.15 & 5.62 \\
        Inference Steps     & 1 & 20 & 16 & 20 & 1 & 1 & 1 & 1 & 20 & 20 \\
        Performance & {7518.3}  & 4040.1 & 7103.6 & 5665.9 & 4583.8 & 5530.6 & 2781.9 & 2100.2 & 4206.4 & 5717.9 \\
        \bottomrule
    \end{tabular}
        }
\end{table}

\section{Implementation Details.}
\label{app:implementation_details}

\subsection{Experimental Details}
\paragraph{Environments.}
{We evaluate seven OpenAI Gym MuJoCo tasks~\citep{brockman2016openai,todorov2012mujoco}. In our runs, each method receives 1M ($10^6$) environment steps under a common evaluation protocol. SMFP results are means and standard deviations over 20 seeds; Tab.~\ref{tab:main_results} supplements reproduced baseline scores with comparable published results when stronger (Appendix~\ref{app:baseline_comparison}).}

\begin{figure}[htbp]
    \centering   
    \includegraphics[trim=0 400 0 0, clip, width=1\textwidth]{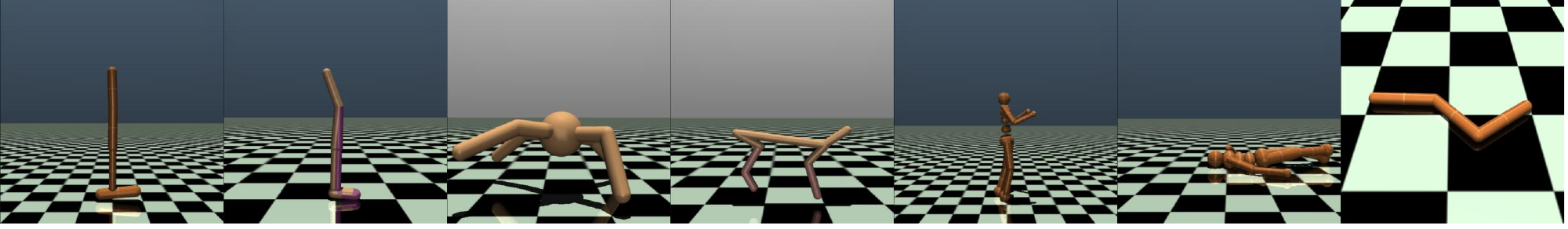} 
    \caption{{MuJoCo environments used in our experiments: Hopper-v4, Walker2d-v4, Ant-v4, HalfCheetah-v4, Humanoid-v4, HumanoidStandup-v4 and Swimmer-v4.}} 
    \label{fig:mujoco} 
\end{figure}

\begin{table*}[htbp]
\vspace{-0.3cm}
\caption{{Comparison of policy representations and optimisation components.} {\textbf{Gaussian}: Uses Gaussian or deterministic components; SMFP uses latent-conditioned Gaussians.} \textbf{Generative}: Uses diffusion or {flow-based policies}. \textbf{SAC-Style}: Uses entropy-regularised updates. {\textbf{MD / guided fitting}: Includes mirror descent and related value-guided policy fitting. FlowRL uses value-weighted flow matching under Wasserstein-2 regularisation.}}
\label{tab:method_comparison}
\resizebox{\textwidth}{!}{%
\begin{tabular}{@{}lccccc@{}}
\toprule
\multirow{2}{*}{\textbf{Family}} & \multirow{2}{*}{\textbf{Algorithm}} & \multicolumn{2}{c}{\textbf{Representation Level}} & \multicolumn{2}{c}{\textbf{Optimisation Level}} \\ \cmidrule(lr){3-4} \cmidrule(l){5-6} 
 &  & \textbf{Gaussian Policy} & \textbf{Generative Policy} & \textbf{SAC (Ent. Reg.)} & {\textbf{MD / guided fitting}} \\ \midrule
\multirow{4}{*}{\textbf{Gaussian}} 
 & PPO~\citep{schulman2017PPO} & \cmark & \xmark & \xmark & \xmark \\
 & SPO~\citep{xie2025simple} & \cmark & \xmark & \xmark & \xmark \\
 & TD3~\citep{fujimoto2018addressing} & \cmark$^{*}$ & \xmark & \xmark & \xmark \\
 & SAC~\citep{haarnoja2018soft} & \cmark & \xmark & \cmark & \xmark \\ \midrule
\multirow{6}{*}{\textbf{Diffusion}} 
 & DIPO~\citep{yang2023dipo} & \xmark & \cmark & \xmark & \cmark \\
 & QSM~\citep{psenka2024qsm} & \xmark & \cmark & \xmark & \cmark \\
 & QVPO~\citep{ding2024diffusion} & \xmark & \cmark & \xmark & \cmark \\
 & DIME~\citep{dime} & \xmark & \cmark & \cmark & \xmark \\
 & MaxEntDP~\citep{MaxEntDP} & \xmark & \cmark & \cmark & \xmark \\
 & DPMD~\citep{maefficientDPMD} & \xmark & \cmark & \xmark & \cmark \\ \midrule
\multirow{3}{*}{\textbf{Flow}} 
 & SAC Flow~\citep{zhang2025sac} & \xmark & \cmark & \cmark & \xmark \\
 & FlowRL~\citep{lv2025flow} & \xmark & \cmark & \xmark & \cmark \\
 & FPMD~\citep{chen2025one} & \xmark & \cmark & \xmark & \cmark \\ \midrule
\textbf{Ours} & \textbf{SMFP} & \cmark & \cmark & \cmark & \cmark \\ \bottomrule
\multicolumn{6}{l}{\footnotesize $^{*}$ {TD3 uses a deterministic actor with added Gaussian noise for exploration and target policy smoothing.} }
\end{tabular}%
}
\vspace{-0.3cm}
\end{table*}

\paragraph{Baselines.}  
\label{app:baseline_comparison}
We compare our approach against three categories of online RL algorithms: (1) Classical Model-Free RL: Methods employing Gaussian or deterministic policies with efficient single-step inference ($1\text{-NFE}$), including PPO~\citep{schulman2017PPO} (\url{https://github.com/ikostrikov/pytorch-a2c-ppo-acktr-gail}), SPO~\citep{xie2025simple} (\url{https://github.com/MyRepositories-hub/Simple-Policy-Optimization}), TD3~\citep{fujimoto2018addressing} (\url{https://github.com/sfujim/TD3}), and SAC~\citep{haarnoja2018soft} (\url{https://github.com/toshikwa/soft-actor-critic.pytorch}). (2) Diffusion-Based RL: Generative policy algorithms such as DIPO~\citep{yang2023dipo} (\url{https://github.com/BellmanTimeHut/DIPO}), QSM~\citep{psenka2024qsm} (\mbox{\url{https://github.com/escontra/score_matching_rl}}), QVPO~\citep{ding2024diffusion} (\mbox{\url{https://github.com/wadx2019/qvpo}}), DIME~\citep{dime} (\url{https://github.com/ALRhub/DIME}), MaxEntDP~\citep{MaxEntDP} (\url{https://github.com/diffusionyes/MaxEntDP}), and DPMD~\citep{maefficientDPMD} (\url{https://github.com/mahaitongdae/diffusion_policy_online_rl}). These methods typically involve iterative sampling, though some are optimised for efficiency. (3) Flow-Based RL: Emerging flow-matching approaches, including SAC Flow~\citep{zhang2025sac} (\url{https://github.com/Elessar123/SAC-FLOW}), FlowRL~\citep{lv2025flow} (\url{https://github.com/bytedance/FlowRL}), and FPMD~\citep{chen2025one}. 
{Tab.~\ref{tab:method_comparison} compares these baselines. We use their published hyperparameters as the starting point for reproduction and tuning.} 

{We tuned reproduced baselines, including on environments absent from their original evaluations. For each baseline and environment in Tab.~\ref{tab:main_results}, we report the stronger of our reproduction and a comparable published or benchmark result, where available. SMFP entries use our 20-seed evaluation. The learning curves in Fig.~\ref{fig:exp_result_curve} use reproduced runs only, with a common training budget and evaluation protocol; table values and curves can therefore differ.}

\paragraph{{Critic network architecture.}}
{Following MeanFlowQL~\citep{meanflowql}, we use multilayer perceptrons (MLPs) with four hidden layers of 512 units each for the critics.}
To improve training stability, we also apply {layer normalisation} to the value networks.

\paragraph{{Actor network architecture.}}
{The actor is a shallow Transformer adapted from DiT~\citep{dit}, with 3 layers, 2 attention heads, and hidden size $256$. Feature embeddings encode the observation and action inputs, while time and positional embeddings provide the conditioning used in the MeanFlow objective and its Jacobian-vector products. We use residual mixing with scale 0.1, gradient clipping at norm 1.0, and a zero-initialised final projection.}

\section{Hyperparameters}
\label{sec:hyper}
{Table~\ref{tab:hyperparameters_general} lists the shared architecture and optimisation settings for the MuJoCo experiments. Table~\ref{tab:task-specific} gives the task-specific mirror-descent coefficient $\lambda$.} 

{Table~\ref{tab:task-search} provides default values and candidate ranges for new environments. We recommend starting from these defaults, then searching the listed candidates where needed. For sequential tuning, fix the entropy coefficient $\alpha$ before varying the mirror-descent coefficient $\lambda$.}

\begin{table}[htbp]
\centering
\caption{Recommended Hyperparameter Search Space for New Environments}
\label{tab:task-search}
\small
\renewcommand{\arraystretch}{1.1}
\begin{threeparttable}
    \begin{tabularx}{0.95\linewidth}{l c >{\raggedright\arraybackslash}X >{\raggedright\arraybackslash}X}
    \toprule
    \textbf{Parameter} & \textbf{Default} & \textbf{Candidate Values} & \textbf{Description} \\
    \midrule
    
    q\_agg & min & $\{\text{min}, \text{mean}\}$ & Target Q aggregation method \\
    
    $\kappa_{\sigma}$ & -3 & $\{-1,-3,-5,-7\}$ & {Mean log-scale threshold} \\
    
    $\alpha$ & 0.2 & $\{0.01, 0.05, 0.1, 0.2,0.4,1\}$ & Entropy temperature in SAC \\
    
    $\lambda$ & 0.3 & $\{0.03, 0.1, 0.3, 1, 3, 10\}$ & Mirror descent regularisation coefficient \\

    num\_candidates ($K_b$) & 8 & $\{1, 4, 8, 16, 32\}$ & {Number of candidates} \\
    \bottomrule
    \end{tabularx}
\end{threeparttable}
\end{table}

\begin{table*}[htbp]
\caption{
\footnotesize
\textbf{Task-Specific Hyperparameter Settings.}
We report the mirror descent regularisation coefficient ($\lambda$) used for each environment.
}
\label{tab:task-specific}
\centering
\vspace{3pt}
\resizebox{\textwidth}{!}{
\begin{threeparttable}
\setlength{\tabcolsep}{6pt} 
\begin{tabular}{lccccccc}
\toprule
\textbf{Environment} & \textsc{Hopper-v4} & \textsc{{Walker2d-v4}} & \textsc{HalfCheetah-v4} & \textsc{Ant-v4} & \textsc{Humanoid-v4} & \textsc{HumanoidStandup-v4} & \textsc{Swimmer-v4} \\
\midrule
\textbf{Lambda ($\lambda$)} & 3 & 3 & 0.3 & 0.3 & 0.3 & 0.3 & 3 \\
\bottomrule
\end{tabular}
\end{threeparttable}
}
\end{table*}

\clearpage
\begin{table}[H]
\centering
\caption{Hyperparameters in Experiments}
\label{tab:hyperparameters_general}
\small
\renewcommand{\arraystretch}{1.15}
\begin{threeparttable}
\begin{tabularx}{\linewidth}{l c L}
\toprule
\textbf{Parameter} & \textbf{Value} & \textbf{Description} \\
\midrule
\rowcolor{mask!88}
\multicolumn{3}{c}{\textbf{General Settings}} \\
batch\_size & 256 & Number of samples per batch \\
q\_agg & min \tnote{+} & Target Q aggregation method \\
noise\_type & gaussian & Noise distribution type \\
optimiser & Adam~\citep{kingma2014adam} & {Network optimiser} \\
gradient steps & 1000000 & Total gradient updates \\
discount $\gamma$ & 0.99 & {Same as QVPO~\citep{ding2024diffusion}} \\
lr & $3\mathrm{e}{-4}$ & Learning rate \\
lr\_schedule & cosine\_with\_warmup & Learning rate schedule \\
lr\_min\_ratio & 0.1 & Minimum ratio in cosine scheduler \\
\rowcolor{mask!88}
\multicolumn{3}{c}{\textbf{Critic Network}} \\
value\_hidden\_dims & (512, 512, 512, 512) & Hidden layer sizes \\
layer\_norm & True & Apply layer normalisation \\

gradient clipping & 1.0 & Maximum gradient norm \\
network initialisation & kaiming\_init & Network initialisation techniques \\

\rowcolor{mask!88}
\multicolumn{3}{c}{\textbf{Actor Network}} \\
actor\_hidden\_dims & 256 & {Policy hidden size} \\
actor\_depth & 3 & Number of transformer layers \\
actor\_num\_heads & 2 & Number of attention heads \\
actor\_layer\_norm & False & Use layer norm in actor \\
discount & 0.99 & Discount factor $\gamma$ \\
tau & 0.005 & Target network update rate \\
tanh\_squash & False & Use tanh at actor output \\
use\_output\_layernorm & False & Norm at final actor output \\
gradient clipping & 1.0 & Maximum gradient norm \\
network initialisation & zero\_init & Network initialisation techniques \\
$\kappa_{\sigma}$ & -3 \tnote{+} & {Mean log-scale threshold} \\
time\_steps & 100 & Time steps of MeanFlow ({same as} \citep{meanflowql}) \\

\rowcolor{mask!88}
\multicolumn{3}{c}{\textbf{{Best-of-$K$ Selection}}} \\
num\_candidates (K) & 8\tnote{+} & {Number of candidates} \\
action\_mode & best & Action selection strategy \\
$N_d$ & 64 & Number of samples from old policy\citep{ding2024diffusion} \\
Transform & $qadv$ & Q-weight transformation function \citep{ding2024diffusion}\\

\rowcolor{mask!88}
\multicolumn{3}{c}{\textbf{SMFP}} \\
$\alpha$ & 0.2 \tnote{+}  & Entropy temperature hyperparameter (SAC) \\
$\lambda$ & Task-dependent &  See Tab.~\ref{tab:task-specific} \\ 
\bottomrule
\end{tabularx} 
\begin{tablenotes}
    \footnotesize
    \item[+] Task-dependent. For new environments, we recommend performing a grid search over the candidate values (refer to Tab.~\ref{tab:task-search}) to achieve optimal performance.
\end{tablenotes}
\end{threeparttable}
\end{table}



\end{document}